\documentclass[twoside]{article}

\usepackage[accepted]{aistats2020}
\input{style.sty}

\begin{document}

%

%

\twocolumn[

\aistatstitle{Online Batch Decision-Making with High-Dimensional Covariates}



\aistatsauthor{Chi-Hua Wang \And Guang Cheng}

\aistatsaddress{ Purdue University \And  Purdue University  } 
 ]

\begin{abstract}
We propose and investigate a class of new algorithms for sequential decision making that interacts with \textit{a batch of users} simultaneously instead of \textit{a user} at each decision epoch. This type of batch models is motivated by interactive marketing and clinical trial, where a group of people are treated simultaneously and the outcomes of the whole group are collected before the next stage of decision. In such a scenario, our goal is to allocate a batch of treatments to maximize treatment efficacy based on observed high-dimensional user covariates. We deliver a solution, named \textit{Teamwork LASSO Bandit algorithm}, that resolves a batch version of explore-exploit dilemma via switching between teamwork stage and selfish stage during the whole decision process. This is made possible based on statistical properties of LASSO estimate of treatment efficacy that adapts to a sequence of batch observations. In general, a rate of optimal allocation condition is proposed to delineate the exploration and exploitation trade-off on the data collection scheme, which is sufficient for LASSO to identify the optimal treatment for observed user covariates. An upper bound on expected cumulative regret of the proposed algorithm is provided.
\end{abstract}

\vspace{-3mm}
\section{Introduction}
\vspace{-1mm}

We consider a high-dimensional online batch decision making problem, a setting in which the decision-maker must interacts with \textit{a group of users}, instead of \textit{a user}, at each decision epoch.
Such setting arises very naturally in real-world applications but received less attention in the literature. In interactive marketing \cite{bertsimas2007learning}, marketers choose among a set of marketing messages to be sent to several customers simultaneously; updating marketing strategy once receiving a feedback is computationally impractical, and 
a more practical approach is to aggregate data in a prescribed length of period before adopting new strategy. In clinical trials \cite{ahuja2016response}, physicians choose among a set of available treatments to be administered to a group of patients simultaneously; updating treatment policy once measuring a response is unrealistic, and therapy in real practice is to collect data in a pre-approved length of period before designing new policy. In  real-world practice, applicability of online decision-making methodology turns out to be impeded by such \textit{limited adaptivity}. Similar issue has been addressed by \cite{bai2019provably} in reinforcement learning setting, but it remains open in the setting of online decision making under bandit feedback.

\begin{figure*}
  \centering \includegraphics[width=0.95\textwidth]{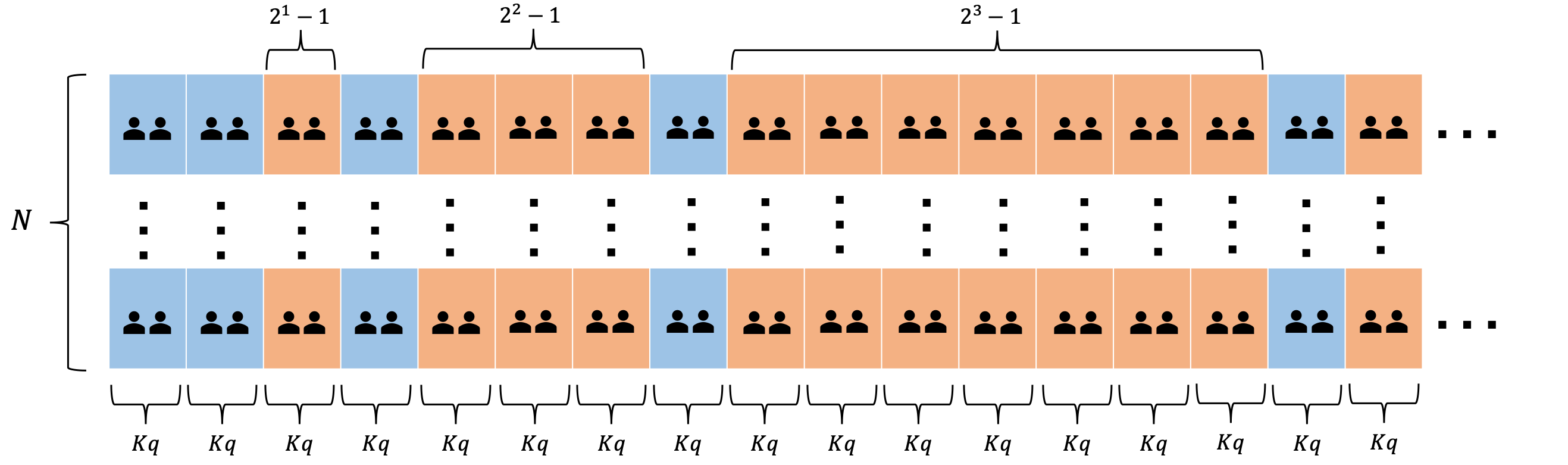}
  \caption{A realization of proposed Teamwork-Selfish policy in online batch decision making. N is the batch size. K is the number of available treatments. q is the number of repetition in a block. The agent runs teamwork mode in blue blocks and runs selfish mode in red blocks. The samples collected from blue blocks is called Teamwork sample. The samples collected from red blocks is called Selfish sample.}
  \label{online_batch_decision_making}
 \vspace{-3mm}
\end{figure*}

One feature of modern economy digitization for decision makers is to deliver {\em personalized} products, services, and solutions based on individual-level data. Additionally, in many practical settings, individual-level data are \textit{high-dimensional} but typically only a small number of the observed covariates are decisive. An additional layer of complexity in online batch decision-making is that, the whole decision making process is learned from {bandit feedback}: decision makers only observe the outcome for the product that was delivered, but not for any other available products that could have been delivered. Taking heterogeneity and bandit feedback into account, approaches of approximate dynamic programming or Markov decision process in \cite{ahuja2016response} addressed the dimensionality issue by transforming the covariate space into a finite number of states and then solving the corresponding Bellman equation. However, how to construct such a transformation is often unspecified in this line of approach and particularly ambiguous in the case of high-dimensional covariate space, impeding itself to embrace the blessing of modern economy digitization. 

In this paper, we propose a new class of approaches, named \textit{Teamwork LASSO Bandit algorithm}, for online batch decision-making to tackle the curse of heterogeneity and high-dimensional covariate space under bandit feedback. To achieve both goals simultaneously, a deliberate balance between exploration and exploitation is required to efficiently learn a personalize decision-making rule that maximizes the cumulative treatment efficacy. In particular, in \textit{Teamwork LASSO Bandit} algorithm, the agent switches between teamwork mode and selfish mode during the whole decision process to alternatively perform pure exploration and pure exploitation. By the proposed Teamwork-Selfish policy, the resulting sample set strikes a balance on optimal allocation in the sense that the rate of optimal allocation is beyond certain level for each treatment (see Definition \ref{templateCond}) with high probability. This strategy ensures that the LASSO regression accommodates dependency arising in a sequence of batch observations and enjoys certain statistical properties in order to achieve optimal allocation for each incoming user.

\textbf{Contributions.} 
Our contribution is three-fold. First, we propose the ``Teamwork LASSO Bandit algorithm,'' that solves the online batch decision-making problem with high-dimensional user covariates via learning LASSO estimates of treatment efficacy. 
Second, we propose a general ``optimal allocation rate condition'' on sample set as a theoretical guideline in designing a data collection scheme to identify the optimal treatment given observed user covariates. Such a scheme is illustrated by the proposed Teamwork-Selfish policy. Last, we establish an upper bound of the expected cumulative regret that scales linearly with the batch size. As a technical by-product, we develop a deviation inequality of LASSO regression that adaptes to a sequence of batch observations. To our knowledge, this is the first work addressing batch version of explore-exploitation dilemma in the setting of online batch decision-making.

\textbf{Related works}
Batched bandit problem has received less research attention over the last decade. The special case of two-armed bandit setting has been addressed in \cite{perchet2016batched}; recently, \cite{gao2019batched} extends the setting to K-armed bandit. While these works addressed the problem by mean model, we addressed batched bandit problem with high-dimensional linear model. In non-batch setting, \cite{rusmevichientong2010linearly} addressed the bandit problem by linear model in low dimensional setting; \cite{bastani2020online} then extends the setting to high-dimensional linear model. Comparing to these works, our model extends such setting to batched bandit problem.

\textbf{A toy example for illustration}

Consider the simplest scenario where the agent allocates one of two available treatments to each user in a size-2 batch at every decision epoch (That is, $N=2, K=2$ in Figure \ref{online_batch_decision_making}). In this case, each block contains $2q$ epochs and each epoch has two user. In the blue block, the agent runs in a teamwork mode: it sends both users to the first treatment at the first $q$ epochs and to the second treatment at the second $q$ epochs; in the red block, the agent runs in a selfish mode: it sends each user to her estimated optimal treatment based on all historical information over all $2q$ epochs. The agent runs in a teamwork mode (blue) on the block whose number is of power of $2$, that is, $1,2,4,8,16,\cdots$ and runs in a selfish mode on the red block. The samples collected from blue blocks are called Teamwork sample; the samples collected from red blocks are called Selfish sample. More specifically, in red block, the agent runs a selfish mode by a two-step procedure: the first step is to run LASSO regression on Teamwork sample to output a set of optimal treatment candidates for the current user; the second step is to run LASSO regression on the union of Teamwork sample and Selfish sample to decide the optimal treatment for the current user.
\paragraph{Notation}
For any positive integer $n$, define $[n] = \{1,\cdots, n\}$ and for any $n_1 < n_2, [n_1:n_2] = \{n_1,\cdots, n_2\}$ and $(n_1:n_2] = \{n_1+1,\cdots, n_2\}$. $|A|$ denotes the number of elements for any collection $A$. For any vector $v$, notation $\text{supp}(v) \equiv \{i| v_i \neq 0\}$ denotes the indexes of non-zero coordinate. For a vector $u=(u_1,\cdots, u_{d})$, $\|u\|_{\infty}\equiv \max_{i \in [d]}|u_i|$ denotes the maximum absolute value of its entries; $\|u\|_{1}\equiv \sum_{i=1}^{d}|u_i|$ the $L_1$-norm. 
\vspace{-3mm}
\section{Online Batch Decision Making}
\vspace{-3mm}

\textbf{Treatment efficacy model}
We consider an agent, who needs to allocate one of $K$ different available treatments (arms), denoted as $\mathcal{W}=\{w_k| k \in[K]\}$, for a size-$N$ batch of users in each decision epoch $t=1,2,\cdots, T,$ where $T$ denotes the length of the decision epoch horizon. The users in the current batch are represented by $N$ {observable} vectors of features (covariates) $x_{i,t}\in \mathcal{X} \subseteq \mathbb{R}^{d}, i \in [N]$. We assume that feature vectors $x_{i,t}$ may vary across decision epoch and are sampled independently from a fixed, but a priori {unknown} distribution $\mathcal{P}_{X}$ with bounded support $\mathcal{X}$.

The efficacy of treatment $w$ for the i\textit{th} user (feedback) at epoch $t$ has value $y(x_{i,t}, w)$, where the function $y$ is unknown. At each decision epoch $t$, the agent allocates a batch of treatments $\{w(x_{i,t}): i\in [N]\}$, 
and then observes a batch of efficacy $\{y(x_{i,t}, w(x_{i,t})): i \in [N]\}$.
The objective is to design an \textit{allocation policy}, which maps users to their own treatments, that maximizes the cumulative sum of observed efficacy.

We assume that the  efficacy of treatment $w$ for user is a linear function of her covariates $x_{i,t}$, namely
\begin{equation}\label{model}
y(x_{i,t}) = 
\langle \beta_{w},
x_{i,t} \rangle +\epsilon_{(i, t)},
\end{equation}
where $\{\epsilon_{(i,t)}:i\in[N], t \in [T]\}$ are martingale difference noise. At each epoch $t$, $\epsilon_{(i,t)}$'s are drawn independently from a mean zero $\sigma$-sub-Gaussian distribution (that is, $E[\exp(\lambda \epsilon_{(i,s)})]\le \exp(\sigma^2\lambda^2/2)$for all real $\lambda$) and
$\{(\epsilon_{i,t}, \mathcal{F}_{t})\}_{t \in [T]}$  forms a martingale difference sequence (that is, $E[\epsilon_{(i,t)}|\mathcal{F}_{t-1}]=0$ for all $i \in [N]$).
The noise is often introduced to account for the features that are not included in the model.

Efficacy parameters  $\mathbf{B}\equiv\{\beta_{w_k}: w \in \mathcal{W}\}$ are a prior unknown to agent. Therefore, the agent deals with exploration-exploitation trade-off as it needs to choose between learning $\mathbf{B}$ and exploiting what has been learned so far to maximize  treatment efficacy.

Our proposed algorithm exploits the structure (sparsity) of the feature space to improve its performance. For this purpose, let $s_{0, w}:= \|\beta_{w}\|_0$ denote the number of nonzero coordinates of $\beta_{w}$. Define $s_{0}=\max_{w \in \mathcal{W}} s_{0, w} $
and note that $s_{0}$ is a priori unknown to the agent.

\paragraph{Technical assumptions}

For ease of presentation, we assume that $\|x_{i,t}\|_{\infty} \le x_{\max}$, for all $x_{i,t} \in \mathcal{X}$, and $\max_{w \in \mathcal{W}}\|\beta_{w}\|_1 \le b$ for a known constant $b$. 

We denote by $\Omega$ the set of feasible parameters, that is,
\begin{equation}\label{feasible_para_set}
\Omega \equiv \{\beta \in \mathbb{R}^{d}: 
\|\beta\|_{0} \le s_{0}, \|\beta\|_1 \le b\},  
\end{equation}
and we write $\mathbf{B} \subseteq \Omega$ if $\beta \in \Omega$ for all $\beta \in \mathbf{B}$.

To measure the performance of proposed bandit algorithm, we present three technical assumptions.

\begin{assumption}{(Margin Condition)}\label{as:Margin_Condition}
There exists a constant $C_0 >0$ such that for $w_i \neq w_j$ in $\mathcal{W}$, $P(0 < |X^\top\beta_{w_i} - X^\top \beta_{w_j}| \le \kappa)
    \le C_0 \kappa$ for all $\kappa >0$.
\end{assumption}
Assumption \ref{as:Margin_Condition} is referred to the Margin Condition in the classification literature \cite{tsybakov2004optimal} and is introduced in multi-armed linear bandit literature to ensure only a small fraction of features can be drawn near the classification boundary $\{x:x^T(\beta_{w_i}-\beta_{w_j})=0\}$ in which efficacy of both treatments are almost equivalent; see \cite{rusmevichientong2010linearly}, \cite{bastani2020online}.

\begin{assumption}{(Treatment Optimality Condition)} \label{as:TreatOptCond}
There exist some constant $h>0$ and two mutually exclusive sets, denoted as $\mathcal{W}_{\text{opt}}$ and $\mathcal{W}_{\text{sub}}$ with $\mathcal{W}=\mathcal{W}_{\text{opt}}\cup \mathcal{W}_{\text{sub}}$ such that
\vspace{-3mm}
\begin{itemize}[noitemsep]
    \item[(a)] For each treatment $w_i$ in $\mathcal{W}_{\text{sub}}$, it holds for every covariate vector $x \in \mathcal{X}$ that
    \begin{equation}
    \langle \beta_{w_i},
    x\rangle< \max_{w \in \mathcal{W}\setminus\{w_{i}\}} \langle \beta_{w},
    x\rangle - h.
    \end{equation}
\vspace{-5mm}
    \item[(b)] For each treatment $w_i$ in $\mathcal{W}_{\text{opt}}$, there exists a constant $p_* > 0$ such that
    \begin{equation}
        \min_{w_i \in \mathcal{W}_{\text{opt}}}P(X \in U_{w_i}) \ge p_*,
    \end{equation}
\vspace{-3mm}
where 
    $$U_{w_i} \equiv \{x \in \mathcal{X}| 
    \langle \beta_{w_i},
    x\rangle > \max_{w \in \mathcal{W}\setminus\{w_{i}\}} 
    \langle \beta_{w},
    x\rangle+h
    \}.$$
\end{itemize}
\vspace{-3mm}
\end{assumption}

Assumption \ref{as:TreatOptCond} is referred to the Treatment Optimality Condition in \cite{rusmevichientong2010linearly}, \cite{bastani2020online} and is to separate available treatments into an optimal subset $\mathcal{K}_{opt}$ and a sub-optimal subset $\mathcal{K}_{sub}$ such that every optimal treatment $w \in \mathcal{K}_{opt}$ is strictly optimal for \textit{some} users (denoted by the set $U_{w}$) and every sub-optimal treatment is strictly sub-optimal for \textit{every} users.

\begin{assumption}
{(Compatibility Condition)}\label{as:CompatabilityCondi}
There exists a constant $\phi_0>0$ such that for each optimal treatment $w \in \mathcal{W}_{opt}$, the population covariance matrix $\Sigma_{w}\equiv E[XX^\top| X \in U_{w}]$ belongs to the compatibility set of its treatment efficacy parameter $\beta_{w}$, that is, 
\begin{equation*}
    \Sigma_{w}
    \in
    \mathcal{C}( \text{supp}(\beta_{w_i}), \phi_0),
\end{equation*}
where $\mathcal{C}(I, \phi)$ is defined as
\begin{equation*}
\begin{aligned}
\mathcal{C}(I, \phi)
    &\equiv
    \{
    M \in \mathbb{R}^{p\times p}_{\succeq 0}
    |\forall v \in \mathbb{R}^{p} \text{ such that} \\
\|v_{I^{c}}\|_1 \le &3\|v_{I}\|_1,
\text{ we have} \|v_{I}\|^2_1\le|I| \frac{v^\top M v}{\phi^2}\}. 
\end{aligned} 
\end{equation*}
\end{assumption}
Assumption \ref{as:CompatabilityCondi} is referred to as the Compatibility Condition in high-dimensional statistics literature \cite{buhlmann2011statistics} and is to ensure that LASSO estimate trained on samples $X \in U_{w}$ converges to the true parameter $\beta_{w}$ with high probability as the number of samples grows to infinity.

\textbf{Oracle Policy and Performance Metric}
We evaluate the performance of our algorithm using the usual notion of regret: the expected treatment efficacy loss compared with the oracle allocation policy that has full knowledge of $\mathbf{B}$, but not of the realizations of noise $\{\epsilon_{(i,t)}: i\in[N], t\in[T]\}$). Let us first characterize this benchmark policy. Based on the model (\ref{model}), the optimal treatment allocation, denoted by $w^{*}$, is defined as
\begin{equation}
    w^{*}(x_{i,t}) = \arg\max_{w \in \mathcal{W}}\{\langle \beta_{w},
x_{i,t} \rangle \}.
\end{equation}
Throughout the paper, $w_{i,t}^{*}$  denotes the optimal treatment allocation for the $i$\textit{th} user at epoch $t$.

We now define regret of a policy. Let $\pi$ be the agent's policy that allocates the treatment $w_{i,t}$ to user $x_{i,t}$, and the choice of $w_{i,t}$ may depend on the information up to decision epoch $t-1$. The worst-case regret is defined as:
\begin{equation*}
\begin{aligned}
&\text{Regret}_{\pi}(T)
\equiv\\
&\max_{\mathbf{B} \in \Omega, \mathbb{P}_{X} \in Q(\mathcal{X})}
E[
\sum_{t=1}^{T}\sum_{i=1}^{N}
(\langle \beta_{w^{*}_{i,t}},
x_{i,t} \rangle 
-
\langle \beta_{w_{i,t}},
x_{i,t} \rangle)
],
\end{aligned}
\end{equation*}
where the expectation is with respect to the distribution of martingale difference noise, $\epsilon_{i,t}$, and the distribution of feature vector, $\mathbb{P}_{X}$. The notation $Q(\mathcal{X})$ denotes a set of probability distributions with bounded support in $\mathcal{X}$. We want to point out that $\text{Regret}_{\pi}(T)$ corresponds to the usual cumulative expected regret in online learning when $N=1$ and the empirical expected estimation error in batch learning when $T=1$.

\textbf{Rate of Optimal Allocation  Condition}
Suppose the agent has collected a sample set $\mathcal{A}_{w} \equiv \{(x_{(i,t)}, y_{(i,t)})\}_{i=1,2,\cdots ; t = 1,2,\cdots}$ for treatment $w$. The set $\mathcal{A}_{w}$ consists of both optimal and sub-optimal allocation. The optimal allocation subsample set is defined as
\begin{equation}
\mathcal{A}^{\sharp}_{w}\equiv \{(\mathbf{X}_{(i,t)}, Y_{(i,t)})\in \mathcal{A}_{w}| \mathbf{X}_{(i,t)} \in  U_{w}\}.
\end{equation}

Intuitively, bigger size of $\mathcal{A}^{\sharp}_{w}$ improves the accuracy of statistical procedure, which is the prize of achieving optimal allocation. On the contrary, bigger size of $\mathcal{A}_{w}$ undermines
the accuracy of statistical procedure, which is
the price of misallocating users to any
suboptimal treatment. The ratio $|\mathcal{A}^{\sharp}_{w}|/|\mathcal{A}_{w}|$
matters and is termed as the \textit{rate of optimal allocation} of sample set $\mathcal{A}_{w}$.

To gain a deeper insight of such balance into high-dimensional setting, consider the LASSO regression:
\begin{equation}\label{optimal allocationLasso}
\widehat{\beta }_{w}(\mathcal{A}_{w}, \lambda)
\equiv\arg\min_{\beta}\left\{ \frac{\|Y-X\beta\|^2_2}{|\mathcal{A}_{w}|}
+\lambda \|\beta\|_1\right\},
\end{equation}
where $Y$ is a $|\mathcal{A}_{w}|$-dimension response vector and $X$ is a $|\mathcal{A}_{w}| \times d$ covariate matrix.

Note that, 
if $\hat{\Sigma}(\mathcal{A}^{\sharp}_{w})$ satisfies the compatability condition with constant $\phi$, then $\hat{\Sigma}(\mathcal{A}_{w})$ satisfies the compatability condition with constant $\phi \sqrt{|\mathcal{A}^{\sharp}_{w}|/|\mathcal{A}_{w}|}$. The above observation suggests that certain balance between $|\mathcal{A}^{\sharp}_{w_k}|$ and $|\mathcal{A}_{w}|$ should be stricken during the decision process. We make this intuition precise for sample sets characterization by introducing the \textit{optimal allocation rate condition}:
\begin{definition}\label{templateCond}
A sample set $\mathcal{A}_{w_k}$ satisfies the \textbf{rate $\mathbf{r}$ optimal allocation condition}, if it satisfies both conditions
\begin{enumerate}[noitemsep]
\item[(i)] Size of sample set: $|\mathcal{A}_{w}| \ge \frac{6 \log d}{r C_2(\phi_1)^2}$
\item[(ii)] Rate of Optimal Allocation :$\frac{|\mathcal{A}^{\sharp}_{w}| }{|\mathcal{A}_{w}| }\ge \frac{r}{2}$.
\end{enumerate}
\end{definition}
We briefly call $\mathcal{A}_{w}$ a \textbf{rate r optimal allocation sample set} if $\mathcal{A}_{w}$ satisfies a optimal allocation condition of certain rate $r$. Note that the optimal allocation subsample set $\mathcal{A}^{\sharp}_{w}$ is not directly observable. The rate $r$ optimal allocation condition is to lower 
bound the size of 
$\mathcal{A}^{\sharp}_{w}$ to ensure enough accuracy of LASSO estimator.

Our first contribution is a deviation inequality for LASSO regression based on the sample set collected in online batch decision making setting:
\begin{theorem}{(\textbf{Deviation inequality for batch-adapted LASSO})}\label{MainOra}
Given a rate r optimal allocation sample set $\mathcal{A}_{w_k}$ follows the dependence structure of online batch decision making problem. If $\lambda = \chi \phi^2/4s_0,$
then for any $\chi > 0$, the oracle inequality holds that
\begin{equation}
\begin{aligned}
&&P(\|\widehat{\beta}(\mathcal{A}_{w_k}, \lambda) - \beta_{\omega_{k}}\|_1>\chi)\\
&\le&
2
\exp[
-C_1(\frac{\phi_1 \sqrt{r}}{2})
|\mathcal{A}_{w_k}|\chi^2
+\log d
]\\
&+&
\exp[
-|\mathcal{A}_{w_k}^{\sharp}|C_2(\phi_1)^2
].
\end{aligned}
\end{equation}
\end{theorem}

Theorem \ref{MainOra} is a general version of LASSO deviation inequality for adapted observations
(see Proposition 1 in \cite{bastani2020online}). Our contribution is to extend the deviation inequality
for
adapted
sequence of batch observations with martingale difference noise.

\section{Teamwork LASSO Bandit Algorithm}
\vspace{-3mm}
In this section, we propose the Teamwork LASSO Bandit algorithm that runs in an interactive fashion: the agent switches between \textit{Teamwork} mode (for pure exploration) and \textit{Selfish} mode (for pure exploitation). The collection of decision epochs that the agent runs in teamwork mode and selfish mode are called teamwork stage and selfish stage, respectively.
In the former stage, the agent allocates the current batch of users to a prescribed and possibly sub-optimal treatment for studying treatment efficacy. In the latter stage, the agent allocates each user in the current batch to her estimated optimal treatment for maximizing treatment efficacy. Epochs in teamwork and selfish stages are collected into sets $\mathbb{T}$ and $\mathbb{T}^{c}(\equiv  [T] \setminus \mathbb{T})$, respectively. 
\vspace{-3mm}
\subsection{Allocation in Teamwork mode}
\vspace{-2mm}
The teamwork stage $\mathbb{T}$ consists of exploration on all available treatments; that is, 
$\mathbb{T} = \cup_{w_{k}\in\mathcal{W}}\mathbb{T}_{\cdot, w_k},$
where $\mathbb{T}_{\cdot, w_k}$
denotes the prescribed decision epochs for studying efficacy of treatment $w_k$.
We defer the exact specification of $\mathbb{T}_{\cdot, w_k}$ to Section \ref{TeamworkPlanning}.

Our treatment allocation policy now runs in Teamwork mode that, in any decision epoch $t \in \mathbb{T}_{\cdot, w_k}$, the agent allocates the whole batch of users to the  treatment $w_k$:  
\begin{equation*}\label{policy:teamwork}
\pi(x_{i,t}) \equiv w_k
\text{   if   } t \in \mathbb{T}_{\cdot, w_k}~~(\text{Teamwork Stage}),
\end{equation*}
for all  $i \in [N]$. Such an allocation results in a batch of covariate-efficacy pairs, $\mathcal{T}_{t, w_k}\equiv \{(x_{i,t}, y(x_{i,t}, w_k)): i \in [N]\}$, which is aggregated into the {teamwork sample set up to epoch t} $\mathcal{T}_{[t], w_k} \equiv \mathcal{T}_{[t-1], w_k}\cup \mathcal{T}_{t, w_k}$.

\vspace{-3mm}

\subsection{Allocation in Selfish mode}
\vspace{-2mm}

The selfish stage $\mathbb{T}^{c}$
deploys exploitation by allocating the current best-possible treatment to each user in the current batch. This is done by a two-step procedure: estimating candidates of personal optimal treatment and then committing selfish allocation. 
Such an allocation results in another batch of covariate-efficacy pairs, $\{(x_{i,t}, y(x_{i,t}, w_{i,t})): i \in [N]\}$. Then, for each treatment $w_k$, the set $\mathcal{E}_{t, w_k}\equiv \{(x_{i,t}, y(x_{i,t}, w_{i,t})): i \in [N], w_{i,t}=w_{k}\}$ is aggregated into the selfish sample set up to epoch t  $\mathcal{E}_{[t],w_k}\equiv \mathcal{E}_{[t-1],w_k}\cup \mathcal{E}_{t,w_k}$.

Here we outline the above two-step procedure whose details are deferred to Section \ref{Selfish2step}. At a selfish decision epoch $t \in \mathbb{T}^{c}$, the agent first estimates $\beta_{w}$ by using LASSO regression based on teamwork sample set $\mathcal{T}_{[t], w}$ and the resulting estimator $\widehat{\beta}_{w}(\mathcal{T}_{[t], w})$
is called Teamwork LASSO. The Teamwork LASSO screens out the sub-optimal treatments for each user $x_{i,t}$ in current batch and then outputs a set of optimal treatment candidates $\hat{\mathcal{K}}(x_{i,t})$(See Eq.\eqref{OptCandidate} for detailed description). 

The agent then estimates $\beta_{w}$ by using LASSO regression based on full sample set $(\mathcal{T} \cup \mathcal{E})_{[t], w}$ and the resulting estimator $\widehat{\beta}_{w}((\mathcal{T}\cup \mathcal{E})_{[t], w})$
is called All LASSO. The All LASSO then  selects the  treatment with highest expected efficacy for the user $x_{i,t}$:
\begin{equation*}\label{policy:selfish}
\begin{aligned}
\pi(x_{i,t}) &\equiv& \arg\max_{w_k\in\widehat{\mathcal{K}}(x_{i,t})}
\langle 
\widehat{\beta}_{w_k}((\mathcal{T}\cup \mathcal{E})_{[t], w})
,
x_{i,t}
\rangle
\\
&\text{   if   }& t 
\notin 
\cup_{w_k \in \mathcal{W}} \mathbb{T}_{\cdot, w_k}
~~(\text{Selfish Stage}).
\end{aligned}
\end{equation*}

\begin{algorithm}[t]
  \caption{\texttt{TeamworkLASSOBandit}$(q, h)$}
  \label{alg:blockstab}
  Given parameters $q$, $h$, $\lambda_1$, $\lambda_{2,0}$
  
  Initialize $\mathcal{T}_{[0],t}, \mathcal{E}_{[0],t}$
  
  \For{$t \in \{1,\ldots, T\}$} {
    \For{$w \in \mathcal{W}$}{
        $\mathcal{T}_{[t], w}=\mathcal{T}_{[t-1], w}$
        ;
        $\mathcal{E}_{[t], w}=\mathcal{E}_{[t-1], w}$
     }
    \eIf{$t \in \mathbb{T}$
    }{
     
    \If{$t \in \mathbb{T}_{\cdot, w}$}{
    
    \For{$i \in \{1,\ldots, N\}$}{
    Observe $x_{i,t}$
    
    Allocate $w_{i,t} = w$
    
    Observe $y(x_{i,t}, w)$
    
    $\mathcal{T}_{[t], w}=\mathcal{T}_{[t], w}\cup \{(x_{i,t},y(x_{i,t}, w))\}$
    
        }
    
    }
    }{
    \For{$w \in \mathcal{W}$}{
    Compute  $\widehat{\beta}_{w}(\mathcal{T})=\widehat{\beta}_{w}(\mathcal{T}_{[t-1], w})$
    
    Compute 
    $\widehat{\beta}_{w}(\mathcal{S})=\widehat{\beta}_{w}((\mathcal{T}\cup \mathcal{E})_{[t-1], w})$
    }
    \For{$i \in \{1,\ldots, N\}$} {
      Observe $x_{i,t}$
      
      Select $\widehat{K}(x_{i,t})
      =\{w|
      \langle x_{i,t},
      \widehat{\beta}_{w}(\mathcal{T}) \rangle 
      \ge 
      \max_{\tilde{w} \in \mathcal{K} \setminus \{w\}}
      \langle x_{i,t},
      \widehat{\beta}_{\tilde{w}}(\mathcal{T}) \rangle 
      -\frac{h}{2}
      \}$
      
      Allocate $w_{i,t}
      = \arg\max_{w \in \widehat{K}(x_{i,t})}
      \langle x_{i,t},
      \widehat{\beta}_{w}(\mathcal{S}) \rangle 
      $
      
      Observe $y(x_{i,t}, w_{i,t})$
      
      $\mathcal{E}_{[t], w_{i,t}}=\mathcal{E}_{[t], w_{i,t}}\cup \{(x_{i,t},y(x_{i,t}, w_{i,t}))\}$
    }
    }
  }
\end{algorithm}
\vspace{-3mm}

Observe that by design of our interactive policy (see Figure \ref{TSpolicy_flowchart}), the agent maintains for each treatment $w$ two different sample sets (teamwork and greedy). The independence in teamwork sample set is preserved by the pure exploration nature of teamwork stage, facilitating the subsequent analysis of LASSO estimate. However, the full sample set mixes independence from teamwork stage with the dependency arising in the selfish stage, complicating the subsequent analysis of LASSO estimate. Tackling such dependency requires a closer look at every epoch of decision making process.

\section{A Teamwork-Selfish Policy}\label{sec:TSpolicy}
\vspace{-3mm}

\subsection{Planing Teamwork Stage}\label{TeamworkPlanning}
\vspace{-2mm}

One factor contributes to 
the success of our Teamwork-Selfish policy is the design of $\mathbb{T}_{[t], w_k}$ in teamwork mode (See eq.\eqref{policy:teamwork}). Here we explain our design of $\mathbb{T}_{[t], w_k}$ and reveal the quantity $q$ that characterized the complexity of online batch decision making problem.

Recall that the collection of teamwork decision epochs up to epoch t is the set $\mathbb{T}_{[t],w_k}$.
Such set is a truncation to epoch range $[t]$ of the collection of treatment $w_k$ teamwork decision epoch $\mathbb{T}_{\cdot,w_k},$ which is developed by the \textit{teamwork rounds} $\{\mathbb{T}_{n,w_k}\}_{n=1}^{\infty}$.

The \textit{$n$th teamwork round for treatment $w_k$} $\mathbb{T}_{n, w_k}$ 
is defined as a prescribed range of decision epochs
\begin{equation}
    \mathbb{T}_{n, w_k} \equiv \{(2^n-1)\times Kq +j| j \in [q(k-1)+1,qk]\}.
\end{equation}
(Refer to the blue block in Figure.\ref{online_batch_decision_making}) Observe that each teamwork round prescribes $q$ decision epochs for treatment pure exploration. The magnitude $q$ characterizes the complexity of a online batch decision problem.

Our theoretical results suggests a lower bound of $q$ by $4 \lceil q_0 \rceil$(See Eq.\eqref{q_zero} for $q_0$). Such lower bound ensures both teamwork sample set 
and full sample set 
collected from the Teamwork-Selfish policy 
to satisfy the rate of optimal allocation  condition (Definition \ref{templateCond}) with high probability. Consequently, the Teamwork LASSO 
and ALL LASSO 
enjoy their statistical properties and the regret bound (\ref{RegretBound}) can be guaranteed.

Remark that our lower bound on $q$ is proportional to $1/N$,  where $N$ is the batch size . Compare to the full adaptive setting in \cite{bastani2020online}, we have $q_{\text{batch}} \approx q_{\text{non-batch}}/N$. In terms of update frequency, LASSO Bandit requires $Kq_{\text{non-batch}}(NT-log NT)$, while \texttt{Teamwork LASSO Bandit} requires $K(q_{\text{non-batch}}/N)(T-log T)$. In terms of regret, LASSO Bandit is of rate $K[log(NT)]^2$, while \texttt{Teamwork LASSO Bandit} is of rate $ KN(logT)^2$. Thus, there is a trade off between regret and update frequency.

Our design of teamwork stage $\mathbb{T}_{\cdot, w}$ is a generalization of the forced sampling  for the two-arm and non-batch setting in \cite{rusmevichientong2010linearly} and  for the $K$-arm and non-batch setting in \cite{bastani2020online}. 
Our contribution is to redesign such type of policy when batch is the fundamental sampling unit to accommodate the restrictions from limited adaptivity.

\begin{figure}
    \centering
    \includegraphics[width=0.5\textwidth]{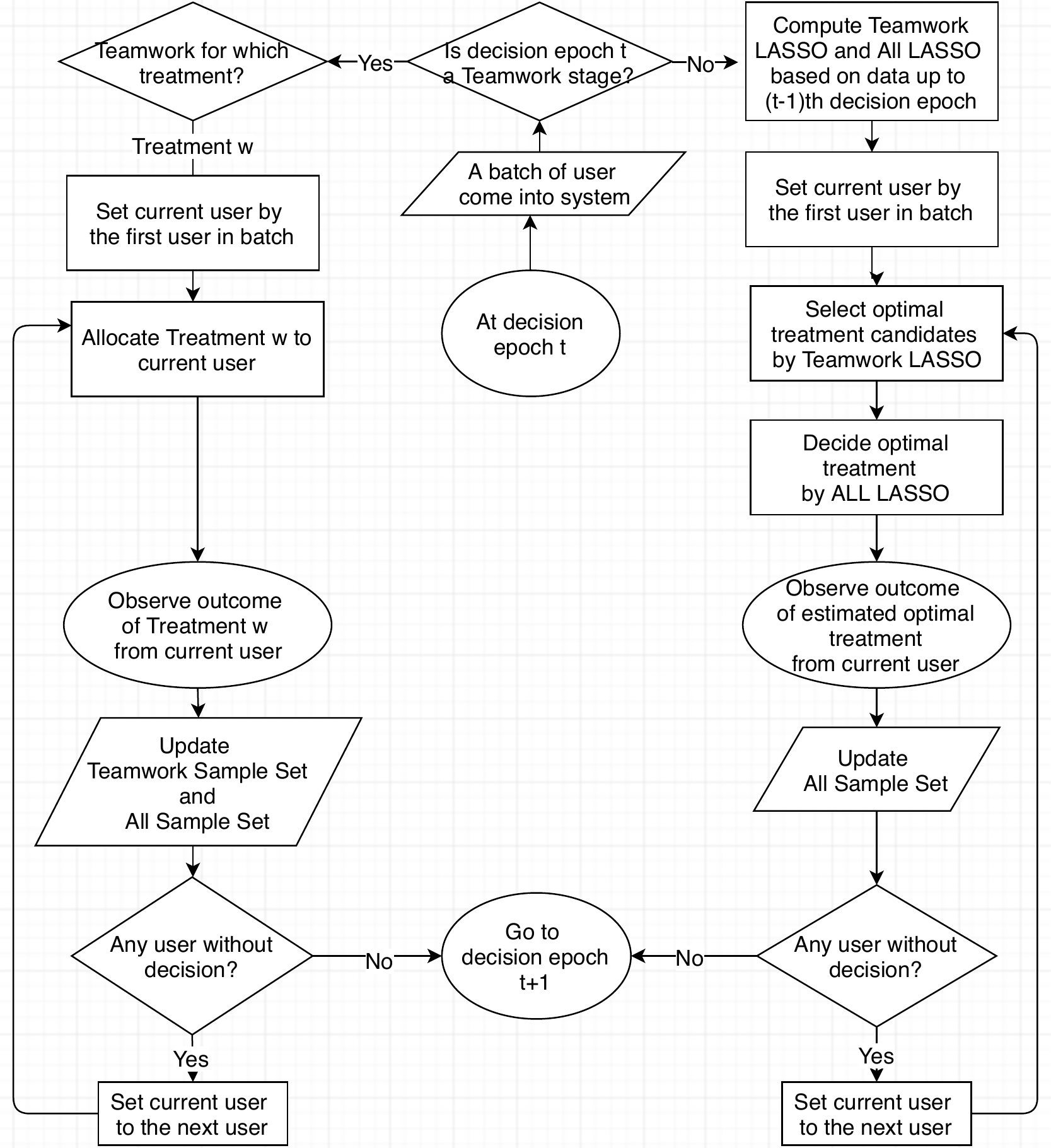}
    \caption{Teamwork-Selfish policy.}
    \label{TSpolicy_flowchart}
\vspace{-4mm}
\end{figure}

\vspace{-3mm}

\subsection{Two Step Procedure in Selfish Stage}\label{Selfish2step}
\vspace{-2mm}

Another factor contributing to 
the success of our Teamwork-Selfish policy is the two step procedure to commit selfish allocation.
In a selfish decision epoch $t$, the agent first estimates every treatment's efficacy based on information up to decision epoch $(t-1)$ by Teamwork LASSO $\widehat{\beta}_{w}(\mathcal{T}_{[t-1], w}, \lambda_1)$
to separate optimal treatments from suboptimal treatments.
Then, the agent estimates candidate treatments' efficacy by ALL LASSO $\widehat{\beta}_{w}((\mathcal{T}\cup 
\mathcal{E})_{[t-1], w}
, \lambda_{2,t-1})$ to identify the true optimal treatment for current user. In particular, the success of step 1 relies on a "good event`` $E_{t}$, defied as
\begin{equation*}\label{goodEvent}
E_{t} \equiv \cap_{w_k \in \mathcal{W}} \{\|\widehat{\beta}_{w_k}(\mathcal{T}_{[t],k},\lambda_1) - \beta_{w_k}\|_1 \le \frac{h}{4 x_{\max}} \},
\end{equation*}
which marks the case that every Teamwork LASSO are accurate enough to screen out suboptimal treatments. 

Now we give a high-probability statement for $E_{t-1}$ to access the performance of Step 1.
\begin{lemma}\label{eventE}
For $t \ge (Kq)^2$, 
    $P(E_{t}) \ge 1-5K/t^4$.
\end{lemma} 
Lemma \ref{eventE} is 
an immediate consequence of Corollary \ref{TeamworkOracle} 

\begin{corollary}{(Deviation inequality for Teamwork LASSO)}\label{TeamworkOracle} For all treatments $w_k \in \mathcal{W}$, if $t \ge (Kq)^2$ for $q \ge 4\lceil q_0 \rceil$, the Teamwork LASSO estimator satisfies the deviation inequality (set $\lambda_1 = \frac{\phi_0^2p_{*}h}{64s_0x_{\max}}$)
\begin{equation*}
    P(\|
    \hat{\beta}_{w_k}(\mathcal{T}_{[t],w_k}, \lambda_1)
    -
    \beta_{w_k}\|_1 
    > 
    \frac{h}{4x_{\max}})
    \le \frac{5}{t^4},
\end{equation*}
where 
\begin{equation}\label{q_zero}
q_{ 0 } \asymp \frac{1}{N}\max\left\{\frac{\log d}{p_{*}} , \frac{x_{ \max }^{ 2 } \log d }{ h ^ { 2 } p _ { * }^ { 2 }} \right\}
\end{equation}
\end{corollary}

Corollary \ref{eventE} is an 
application of Theorem \ref{MainOra}, given a deviation inequality of LASSO regression based on $\mathcal{T}_{[t], w_k}$. As shown in lemma \ref{TeamTemplate},
$\mathcal{T}_{[t], w_k}$ is a rate $p_{*}$ optimal allocation sample set with probability at least $1-\exp(2/t^4)$.

\textbf{Step 1: Screen out Sub-optimal Treatments} 

Given a current user $x$, the agent's objective at step 1 is to output a set of user's potential optimal treatments. To construct such set of treatment candidates, we require that a candidate treatment should have its estimated treatment efficacy almost as good as the maximum of estimated treatment efficacy of all available treatments, as the Teamwork LASSO can tell.

Formulating such intuition motivates us to define the set of \textit{personal optimal treatment candidates} $\widehat{\mathcal{K}}(x)$: 
\begin{equation}\label{OptCandidate}
\begin{aligned}
\widehat{\mathcal{K}}(x) &\equiv\{w_k \in \mathcal{W}:
\langle 
x , \hat{\beta}_{w_k}(\mathcal{T}_{[t-1],w_k},\lambda_1)
\rangle \\
&\ge \max_{w_k \in \mathcal{W}}  \langle x, \hat{\beta}_{w_k}(\mathcal{T}_{[t-1],w_k},\lambda_1)
\rangle- \frac{h}{2}\} 
\end{aligned}
\end{equation}
The exact members in candidate set $\hat{\mathcal{K}}(x)$ depend on the region that $x$ belongs to. Recall that $w^{*}(x) = \arg\max_{w \in \mathcal{W}}\langle x, \beta_{w} \rangle$ denotes the optimal treatment of $x$. For the case $x \in U_{{w}_{k}}$, $w^{*}(x) = {w}_{k}$ by definition of $U_{{w}_{k}}$. As shown in lemma \ref{optCandi}, given the event $E_{t-1}$ holds ,the candidate set $\hat{\mathcal{K}}(x)$ contains \textit{only} the optimal treatment of $x$, that is,
\begin{equation*}
    \hat{\mathcal{K}}(x) = \{ w^{*}(x)\}~~~\text{ if } x \in U_{w_k}.
\end{equation*}
Therefore, the agent definitely has an optimal allocation for every user $x \in U_{w_k}$ under the event $E_{t-1}$.

For the case $x \in \mathcal{X} \setminus \cup_{w \in \mathcal{W}}U_{w}$, the user covariate $x$ lies near a decision boundary $\{x: \langle x,\beta_{w^*(x)}-\beta_{w_j} \rangle =0\}$ for some comparable treatment $w_{j}$. As shown in lemma \ref{allCandi}, the candidate set $\hat{\mathcal{K}}(x)$ contains \textit{at least} the optimal treatment of $x$, that is,
\begin{equation*}
    w^{*}(x)  \in \hat{\mathcal{K}}(x)  ~~~\text{ if } x \in \mathcal{X} \setminus \cup_{w \in \mathcal{W}}U_{w},
\end{equation*}
but may contain other comparable treatments $w_j$ that performs almost equally well as the optimal treatment $w^{*}(x)$, under the event $E_{t-1}$.

\textbf{Step 2: Commit Selfish Allocation}
The agent's objective in step 2 is to commit a selfish allocation for current user $x$. Such selfish allocation is done by allocating user $x$ to the treatment with highest estimated efficacy, as far as the All LASSO can tell:
\begin{equation}
    \pi(x)
    \equiv
    \arg\max_{w_k \in \widehat{K}(x)}
    \langle 
    x,
    \widehat{\beta}_{w_k}(
    (\mathcal{T}\cup\mathcal{E})_{[t-1], w_k}
    ,\lambda_{2,t})
    \rangle
\end{equation}

An application of Theorem \ref{MainOra} gives a deviation inequality of LASSO regression based on $(\mathcal{T}\cup \mathcal{E})_{[t],w_k}$:
\begin{corollary}{(Deviation inequality for All LASSO)}\label{AllOracle} For treatments $w_k \in \mathcal{W}_{\text{opt}}$, if $t \ge C_{5}\equiv \min \{t: t \ge 24Nq\log t + 4(Kq)^2\} $, the All LASSO estimator satisfies the deviation inequality (Set $\lambda_{2,t} = \frac{\phi_0^2}{2s_0}\sqrt{\frac{\log t+\log d}{p_*C_{1}(\phi_0)}\frac{1}{t}}$)
\begin{equation}
\begin{aligned}
&\|    \hat{\beta}_{w_k}((\mathcal{T}\cup \mathcal{E})_{[t],w_k}, \lambda_{2,t})
    -
    \beta_{w_k}\|_1\\
    &>
    \frac{16}{\sqrt{p_*^3 C_1(\phi_0)}}
    \sqrt{
    \frac{\log t + \log d}{t}}
\end{aligned}
\end{equation}
with probability at least $2(
    \frac{1}{t}
    +
    \exp(- \frac{p_*^2 C_2(\phi_0)^2}{32} \cdot t)$.
\end{corollary}
\vspace{-3mm}

\section{Regret Analysis}
\vspace{-2mm}

The following theorem bounds the regret of our Teamwork-Selfish treatment allocation policy.
\begin{theorem}\label{thm:regret_upper}
Suppose Assumptions \ref{as:Margin_Condition}, \ref{as:TreatOptCond} and \ref{as:CompatabilityCondi} hold. Then, the regret of the Teamwork-Selfish policy over decision horizon $[T]$ satisfies
\begin{equation*}
\begin{aligned}
    \text{Regret}_{\pi}(T)
    &\le N\{
    [C_3](\log T)^2\\
    &+
    [2bx_{\max}K(6q +2) + C_3 \log d]
    \log T\\
    &+
    [2bx_{\max}(C_5+K(1+4C_4))]\}
     \label{RegretBound}
    \\
    &=
    O(NKs_0^2\sigma^2[\log T + \log d]^2)
\end{aligned}
\end{equation*}
\end{theorem}

Below we provide a roadmap for the proof of Theorem \ref{thm:regret_upper}. 
The proof is motivated by the regret analysis in \cite{bastani2020online}. Our contribution is to generalize the approach for regret analysis in online batch decision making setting.

\textbf{1. Regret Guarantee in Good and Bad Epochs}

To reason about different sources of regret contribution, we decompose the decision process into four subcases (i-iv) so that we can examine each one independently:
\vspace{-3mm}

\begin{itemize}[noitemsep]
\item[(i)] During initialization period ($t \le C_5$)
\item After initialization ($t > C_5$)
\begin{itemize}[noitemsep]
\item[(ii)] In teamwork stage ($t \in \mathbb{T}$)
\item In selfish stage ($t \notin \mathbb{T}$)
\begin{itemize}[noitemsep]
\item[(iii)] Event $E_t$ does not hold.
\item[(iv)] Event $E_t$ does hold.
\end{itemize}
\end{itemize}
\end{itemize}
\vspace{-3mm}

As shown in lemma \ref{techRegret} in section \ref{case4regret}, the expected regret in case (iv) is guaranteed to be bounded by the function $f(t)$ as
\begin{equation}\label{GoodEpochRegret}
\begin{aligned}
    f(t) &= 
    [4Kbx_{\max} + C_3(\phi_0, p_{*})\cdot \log d]/t\\
    &+
    8Kbx_{\max}
    \exp[-(p_*^2C_2(\phi_0)^2/32)\cdot t]\\
    &+
    C_3(\phi_0, p_{*})
    (\log t/t).
\end{aligned}
\end{equation} For this reason, we define case (iv) as good epochs 
\begin{equation}
G_{t} = I(t>C_{5}, t \in \mathbb{T}, \text{Event } E_{t} \text{ holds}).
\end{equation}
and then we  interpret lemma \ref{techRegret} in section \ref{case4regret} as
\begin{equation}\label{GoodRegretIneq}
   r_{i,t}I(G_{t})
    \le f(t)I(G_{t}).
\end{equation} 

Outside the good epochs are bad epochs $G_{t}^{c}$. In a bad epoch, allocations cannot be guaranteed to be optimal, due to the pure exploration nature in teamwork mode or the insufficient accuracy of LASSO estimate in selfish mode, resulting in the worst-case regret guarantee $2bx_{\max}$ for cases (i-iii). We interpret the above fact as
\begin{equation}\label{BadRegretIneq}
    r_{i,t}I(G_{t}^{c})
    \le 2bx_{\max}I(G_{t}^{c})
\end{equation}.
\vspace{-3mm}

\textbf{2. Bounding Expected Instantaneous Regret}
Combine \eqref{GoodRegretIneq} and \eqref{BadRegretIneq}, the expected instantaneous regret is upper bounded by
\begin{equation}\label{insta_regret_bound}
\begin{aligned}
E[r_{t}]
&=
E[r_{t} \cdot I(G_{t})]
+
E[r_{t}\cdot I(G_{t}^{c})]\\
&\le
E[f(t)\cdot I(G_{t})]
+
E[2bx_{\max}\cdot I(G_{t}^{c})]\\
&=
f(t)P(G_{t})
+
2bx_{\max}P(G_{t}^{c})
\end{aligned}
\end{equation}

\textbf{3. Bounding $\text{Regret}_{\pi}(T)$} Apply 
\eqref{insta_regret_bound} to gain 
\begin{equation}
\begin{aligned}
&E[\sum_{t=1}^{T}\sum_{i=1}^{N}r_{i,t}]
\le 
N\sum_{t=1}^{T}E[r_{i,t}]\\
&\le 
N\int_{0}^{T}
f(t)dt
+2Nbx_{\max}
\int_{0}^{T}
P(G_{t}^{c})dt
\end{aligned}
\end{equation}
Note that, by lemma \ref{eventE}, we have $P(G_{t}^{c}) \le 5K/t^{4}$.
\vspace{-3mm}
\section{Experiment}
\vspace{-2mm}

We illustrate the trade-off between regret and update frequency by comparing the cumulative regret between \texttt{LASSO Bandit} algorithm (high update frequency) and our \texttt{Teamwork LASSO Bandit} algorithm (low update frequency) in Figure \ref{fig:teamworkLasso_exp} ( Appendix, section \ref{append:experiment}). Here we give remarks on the experiment:
\vspace{-3mm}

\begin{enumerate}
\item 
In terms of the number of updates, say case q=1, LASSO Bandit (N=1) (high update frequency) requires $3[5000/3-log(5000/3,2)] \sim 4968$ updates, Teamwork LASSO Bandit (low update frequency) requires $3[5000/3/4-log(5000/3/4,2)] \sim 1224$ updates for N=4 case and $3[5000/3/12-log(5000/3/12,2)] \sim 396$ updates for N=12 case. Note that both give comparable regrets.
\item 
If the length of exploration phase q is sufficiently large, high update frequency algorithm has lower cumulative regret than low update frequency algorithm; if q is not sufficiently large, low update frequency algorithm outperforms high update frequency algorithm.
\item 
In general, the performance of high update frequency algorithm has higher variance than the performance of low update frequency algorithm. In particular, the performance of high update frequency algorithm is more sensitive to the increase in covariate dimension than our low update frequency algorithm.
\end{enumerate}

In conclusion, high update frequency algorithm (\texttt{Lasso Bandit}) do have lower cumulative regret than low update frequency algorithm ( \texttt{Teamwork Lasso Bandit}) if the length of exploration phase q is sufficiently large. However, it is hard to determine how large of q can be thought of as being sufficiently large in practice. On the other hand, low update frequency algorithm is immune from such a concern in the sense that we can simply set q=1 when we have a batch of new samples at every decision epoch.

\vspace{-5mm}
\section{Discussion and Conclusions}
\vspace{-3mm}

We have proposed a framework to address batch-version explore-exploit dilemma in the setting of online batch decision making with high dimensional covariate. In terms of regret analysis, we formulate the \textit{rate of optimal allocation condition} on the collected sample set to characterize the underlying constraint behind the data collection scheme and to serve as a guideline for designing policy in bandit algorithms. Based on the rate of optimal allocation condition, we propose the \texttt{Teamwork LASSO Bandit} algorithm for sequential decision making. In theory, the cumulative total regret of the \texttt{Teamwork LASSO Bandit} algorithm of constant batch size $N$ over finite time horizon $T$ is shown to be bounded by $O(N(\log T)^2)$. In terms of observed covariate dimension $p$ and sparsity parameter $s_0$, the cumulative total regret of the \texttt{Teamwork LASSO Bandit} algorithm grows as $O(s_0^2(\log p)^2)$.

In the end, we highlight a few particularly relevant questions that are left as future works. The first one is the minimax lower bound of regret over all possible algorithms solving batched bandit problem with covariates. In one pull situation, \cite{rusmevichientong2010linearly} showed the lower bound is $O(\log T)$. Recently, \cite{wang2018minimax} showed that the regret of \cite{bastani2020online} can be reduced from $O((\log T)^2)$ to $O(\log T)$ by invoking the Minimax Concave Penalized (MCP) penalty. Hence, the MCP-Bandit algorithm matches the oracle policy with high probability. We expect the regret of \texttt{Teamwork LASSO Bandit} algorithm can also be reduced from $O(N(\log T)^2)$ to $O(N\log T)$ if MCP penalty instead of the lasso one is adopted, although whether this rate matches the theoretical lower bound remains unknown. The second one is more relevant to practice: can we design an effective teamwork strategy when batch size is non-constant? In \cite{grover2018best}, the authors have proposed four different kinds of delayed feedback mechanism that frequently happen in online advertising context, which may lead to non-constant batch size in our setting. When we are performing batch update in the above delayed feedback scenario, is there a guideline for algorithm design? In particular, can the rate of optimal allocation condition be extended to handle delayed feedback situation?

\subsubsection*{Acknowledgements}

Chi-Hua Wang thanks Yang Yu (Purdue University) for discussion 
and Zhanyu Wang (Purdue University) for help on experiment.
Guang Cheng would like to acknowledge support by NSF DMS-1712907, DMS-1811812, DMS-1821183, and Office of Naval Research (ONR N00014- 18-2759). While completing this work, Guang Cheng was a member of Institute for Advanced Study, Princeton and visiting Fellow of SAMSI for the Deep Learning Program in the Fall of 2019; he would like to thank both Institutes for their hospitality.

\nocite*{}
\bibliography{my_ref}

\begin{thebibliography}{}

\bibitem[Abbasi-Yadkori et~al., 2012]{abbasi2012online}
Abbasi-Yadkori, Y., Pal, D., and Szepesvari, C. (2012).
\newblock Online-to-confidence-set conversions and application to sparse
  stochastic bandits.
\newblock In {\em Artificial Intelligence and Statistics}, pages 1--9.

\bibitem[Ahuja and Birge, 2016]{ahuja2016response}
Ahuja, V. and Birge, J.~R. (2016).
\newblock Response-adaptive designs for clinical trials: Simultaneous learning
  from multiple patients.
\newblock {\em European Journal of Operational Research}, 248(2):619--633.

\bibitem[Alon and Spencer, 2004]{alon2004probabilistic}
Alon, N. and Spencer, J.~H. (2004).
\newblock {\em The probabilistic method}.
\newblock John Wiley \& Sons.

\bibitem[Auer, 2002]{auer2002using}
Auer, P. (2002).
\newblock Using confidence bounds for exploitation-exploration trade-offs.
\newblock {\em Journal of Machine Learning Research}, 3(Nov):397--422.

\bibitem[Bai et~al., 2019]{bai2019provably}
Bai, Y., Xie, T., Jiang, N., and Wang, Y.-X. (2019).
\newblock Provably efficient q-learning with low switching cost.
\newblock {\em arXiv preprint arXiv:1905.12849}.

\bibitem[Bastani and Bayati, 2020]{bastani2020online}
Bastani, H. and Bayati, M. (2020).
\newblock Online decision making with high-dimensional covariates.
\newblock {\em Operations Research}, 68(1):276--294.

\bibitem[Bertsimas and Mersereau, 2007]{bertsimas2007learning}
Bertsimas, D. and Mersereau, A.~J. (2007).
\newblock A learning approach for interactive marketing to a customer segment.
\newblock {\em Operations Research}, 55(6):1120--1135.

\bibitem[Bolton and Harris, 1999]{bolton1999strategic}
Bolton, P. and Harris, C. (1999).
\newblock Strategic experimentation.
\newblock {\em Econometrica}, 67(2):349--374.

\bibitem[B{\"u}hlmann and Van De~Geer, 2011]{buhlmann2011statistics}
B{\"u}hlmann, P. and Van De~Geer, S. (2011).
\newblock {\em Statistics for high-dimensional data: methods, theory and
  applications}.
\newblock Springer Science \& Business Media.

\bibitem[Dimakopoulou et~al., 2017]{dimakopoulou2017estimation}
Dimakopoulou, M., Athey, S., and Imbens, G. (2017).
\newblock Estimation considerations in contextual bandits.
\newblock {\em arXiv preprint arXiv:1711.07077}.

\bibitem[Gao et~al., 2019]{gao2019batched}
Gao, Z., Han, Y., Ren, Z., and Zhou, Z. (2019).
\newblock Batched multi-armed bandits problem.
\newblock {\em arXiv preprint arXiv:1904.01763}.

\bibitem[Goldenshluger and Zeevi, 2013]{goldenshluger2013linear}
Goldenshluger, A. and Zeevi, A. (2013).
\newblock A linear response bandit problem.
\newblock {\em Stochastic Systems}, 3(1):230--261.

\bibitem[Grover et~al., 2018]{grover2018best}
Grover, A., Markov, T., Attia, P., Jin, N., Perkins, N., Cheong, B., Chen, M.,
  Yang, Z., Harris, S., Chueh, W., et~al. (2018).
\newblock Best arm identification in multi-armed bandits with delayed feedback.
\newblock {\em arXiv preprint arXiv:1803.10937}.

\bibitem[Jun et~al., 2016]{jun2016top}
Jun, K.-S., Jamieson, K.~G., Nowak, R.~D., and Zhu, X. (2016).
\newblock Top arm identification in multi-armed bandits with batch arm pulls.
\newblock In {\em AISTATS}, pages 139--148.

\bibitem[Li et~al., 2010]{li2010contextual}
Li, L., Chu, W., Langford, J., and Schapire, R.~E. (2010).
\newblock A contextual-bandit approach to personalized news article
  recommendation.
\newblock In {\em Proceedings of the 19th international conference on World
  wide web}, pages 661--670.

\bibitem[Perchet et~al., 2016]{perchet2016batched}
Perchet, V., Rigollet, P., Chassang, S., Snowberg, E., et~al. (2016).
\newblock Batched bandit problems.
\newblock {\em The Annals of Statistics}, 44(2):660--681.

\bibitem[Perchet et~al., 2013]{perchet2013multi}
Perchet, V., Rigollet, P., et~al. (2013).
\newblock The multi-armed bandit problem with covariates.
\newblock {\em The Annals of Statistics}, 41(2):693--721.

\bibitem[Rusmevichientong and Tsitsiklis, 2010]{rusmevichientong2010linearly}
Rusmevichientong, P. and Tsitsiklis, J.~N. (2010).
\newblock Linearly parameterized bandits.
\newblock {\em Mathematics of Operations Research}, 35(2):395--411.

\bibitem[Tsybakov et~al., 2004]{tsybakov2004optimal}
Tsybakov, A.~B. et~al. (2004).
\newblock Optimal aggregation of classifiers in statistical learning.
\newblock {\em The Annals of Statistics}, 32(1):135--166.

\bibitem[Wang et~al., 2018]{wang2018minimax}
Wang, X., Wei, M., and Yao, T. (2018).
\newblock Minimax concave penalized multi-armed bandit model with
  high-dimensional covariates.
\newblock In {\em International Conference on Machine Learning}, pages
  5200--5208.

\end{thebibliography}

\onecolumn
\appendix

\begin{center}
    \Large Supplement to ``Online Batch Decision-Making with High-Dimensional Covariates''
\end{center}

\section{Proof of Theorem \ref{MainOra}}

\textit{Proof of Theorem \ref{MainOra}.}
Based on standard arguments in high dimensional statistics, the Template LASSO on $\mathcal{A}_{w_k}$, when choosing $\lambda \ge 2 \lambda_0(\gamma)$, satisfies 
\begin{equation}\label{Ora01}
\begin{aligned}
P(\|\widehat{\beta}(\mathcal{A}_{w_k}, \lambda) - \beta_{w_k}\|_1 \le 4\lambda\frac{|S_{w_{k}}|}{\phi^2})  
\ge
P[\mathcal{F}(\lambda_0(\gamma))] 
-P(\widehat{\Sigma}(\mathcal{A}_{w_k}) \notin C(S_{w_k},\phi)))
\end{aligned}
\end{equation}
where 
$\mathcal{F}(\lambda_0(\gamma))
\equiv \left\{ \max _ { 1 \leq j \leq p } \frac { 2 } { T } \left| \eta ^ { \top } X ^ { ( j ) }  \right|\leq \lambda _ { 0 } ( \gamma ) \right\},$ which is a high-probability event by carefully choosing the threshold $\lambda_0(\gamma)$ stated in the following lemma:

\begin{lemma}\label{EPbatch}
Given a sample set $\mathcal{A}_{w_k}$ and choose $\lambda_0(\gamma) \equiv 2\sigma x_{\max}\sqrt{(\frac{\gamma^2+2\log p)}{|\mathcal{A}_{w_k}|}}$, then
\begin{equation}
    P(\mathcal{F}(\lambda_0(\gamma))) \ge 1-2\exp[-\frac{\gamma^2}{2}]
\end{equation}
\begin{proof}
See Section \ref{Sec.BatchEP}.
\end{proof}
\end{lemma}

Besides, the sample covariance matrix of the template sample set $\mathcal{A}_{w_k}$ satisfies the compatibility condition with high probability, as stated in the following lemma:
\begin{lemma}\label{SamCompa}
Given a sample set $\mathcal{A}_{w_k}$ that satisfies rate $r$ optimal allocation condition. Then
\begin{equation}
\begin{aligned}
&P(\widehat{\Sigma}(\mathcal{A}_{w_k}) \notin 
\mathcal{C}(\beta_{S_{w_k}}, \frac{\phi}{\sqrt{2}}\sqrt{\frac{|\mathcal{A}_{w_k}^{\sharp}|}{|\mathcal{A}_{w_k}|}}))
\le \exp(-C_2(\phi_1)^2|\mathcal{A}_{w_k}^{\sharp}|).
\end{aligned}
\end{equation}
\end{lemma}
\begin{proof}
See Section \ref{sec.SamComp}.
\end{proof}

Now, lemma \ref{EPbatch} and lemma \ref{SamCompa} together turn the equation (\ref{Ora01}) into
\begin{equation}
\begin{aligned}
P(\|\widehat{\beta}(\mathcal{A}_{w_k}, \lambda) - \beta_{w_k}\|_1 \le 4\lambda\frac{|S_{w_{k}}|}{\phi^2}) \nonumber \ge  1-2\exp[-\frac{\gamma^2}{2}]-\exp(-C_2(\phi_1)^2|\mathcal{A}_{w_k}^{\sharp}|).
\end{aligned}
\end{equation}
Then Theorem \ref{MainOra} follows by solving $\gamma$ from the condition $\lambda \ge 2 \lambda_0(\gamma)$.

\section{Proof of key Lemmas}

\subsection{Proof of Lemma \ref{EPbatch}}
\label{Sec.BatchEP}

\paragraph{Lemma \ref{EPbatch}}
The event 
\begin{equation}
\mathcal{F}(\lambda_0(\gamma))
=
\{\max_{j \in [p]}\frac{2}{T}|\eta^\top X^{(j)}| \le \lambda_0(\gamma)\}.
\end{equation}
holds with probability at least $1-\exp(-\frac{\gamma^2}{2})$ by choosing
\begin{equation}
    \lambda_0(\gamma) = 2x_{\max}\sigma \sqrt{\frac{\gamma^2 + 2 \log d}{\sum_{s=1}^{T}n_{s,w_k}}}.
\end{equation}
\begin{proof}
Let $n_{s,w_k}$ denote the number of users allocated to treatment $w_k$ at the $s$th decision epoch. The sample collected at epoch $s$ is denoted by $\mathcal{A}_{s,k} = \{((X_{(i,s)},Y_{(i,s)}): i \in[n_{t, w_k}], s\in [t])\}$.
Recall $X^{(j)}$ is the jth column of covariate matrix $X$ and the good event
\begin{equation}
\mathcal{F}(\lambda_0(\gamma))
=
\{\max_{j \in [p]}\frac{2}{T}|\eta^\top X^{(j)}| \le \lambda_0(\gamma)\}.
\end{equation}

With the help of union bound, we have 
$$
P \left( \mathcal { J } \left( \lambda _ { 0 } ( \gamma ) \right) \right) \geq 1 - \sum _ { j = 1 } ^ { p } P \left( \left| \eta ^ { \top } X ^ { ( j ) } \right| > \frac { T } { 2 } \lambda _ { 0 } ( \gamma ) \right).
$$
The quantity $\eta^\top X^{(j)}$ actually has sub-Gaussian tail.
To see this, define the filtration

\begin{equation}
F_{t} \equiv \{(Y_{(i,j)}, X_{(i,j)})\}_{i\in [t-1], j \in [n_{(t-1),w_k]}}
~~\text{ and }~~
\mathcal{G}_t = \mathcal{F}_t \cup\{X_{(s,i)}\}_{i \in [n_{t, \omega_k}]}
\end{equation}

From the tower property of expectation, independence between $\{\eta_{(s,i)}\}_{i\in[n_{s,w_k}]}$ given $\mathcal{G}_{s}$, sub-Gaussian assumption on $\eta_{(s,i)}$, bounded assumption on covariates $\left\| X _ { ( i , t ) } \right\| _ { \infty } \leq x _{max}$,we have
\begin{eqnarray}
&&E[\exp(u\sum_{s=1}^{n_{s, \omega_k}}\eta_{(s, i)}X_{(s, i),j})|\mathcal{F}_s]\\
&=&
E[
E[\exp(u\sum_{s=1}^{n_{s, \omega_k}}\eta_{(s, i)}X_{(s, i),j})|\mathcal{G}_s]
|\mathcal{F}_s]\\
&=&
E[
\prod_{s=1}^{n_{s, \omega_k}}
E[\exp(u\eta_{(s, i)}X_{(s, i),j})|\mathcal{G}_s]
|\mathcal{F}_s]\\
&\le&
E[
\prod_{s=1}^{n_{s, \omega_k}}
\exp(u^2\frac{(\sigma X_{(s, i),j})^2}{2})
|\mathcal{F}_s]
\le
\prod_{s=1}^{n_{s, \omega_k}}
\exp(u^2\frac{(\sigma x_{\max})^2}{2})\\
&=&
\exp(u^2\frac{(\sqrt{n_{s, \omega_k}}\sigma x_{\max})^2}{2}).
\end{eqnarray}

The above result gives us a bound on the moment generating function of $\eta^\top X^{(j)}$ that 
\begin{eqnarray}
~E[\exp(u(\eta^\top X^{(j)}))]
&=&
E[\exp(u\sum_{t=1}^{T}\sum_{i=1}^{n_{t,\omega_k}}\eta_{(s,i)}X_{(s,i), j})]\\
&\le&
\exp(u^2\frac{(\sqrt{n_{T,\omega_k}}x_{\max}\sigma)^2}{2})E[\exp(u\sum_{t=1}^{T-1}\sum_{i=1}^{n_{t,\omega_k}}\eta_{(s,i)}X_{(s,i), j})]\\
&\le& \cdots
\le 
\exp(u^2\frac{(\sqrt{\sum_{s=1}^{T}n_{s,\omega_k}}x_{\max}\sigma)^2}{2}).
\end{eqnarray}

We find $\eta^\top X^{(j)}$ is $(\sqrt{\sum_{s=1}^{T}n_{s,\omega_k}}x_{\max}\sigma)^2$-sub-Gaussian.  The tail probability bound of sub-Gaussian distribution gives
$$
P \left( \left| \eta ^ { \top } X ^ { ( j ) } \right| > t \right) \leq 2 \exp \left( - \frac { t ^ { 2 } } { 2 \sum_{s=1}^{T}n_{s,w_k} x _ { \max } ^ { 2 } \sigma ^ { 2 } } \right).
$$

Now, to reformat this into a desired tail probability form, we note
\begin{eqnarray}
1-2\exp(-\frac{\gamma^2}{2})
&=&
1-\sum_{j=1}^{d}P(|\eta^\top X^{(j)}|> \frac{\sum_{s=1}^{T}n_{s,w_k}}{2} \lambda_0(\gamma))\\
&\ge&
1- 2 \exp(-\frac{\sum_{s=1}^{T}n_{s,w_k} \lambda_0^2(\gamma)}{8 x^2_{\max}\sigma^2}+\log d).
\end{eqnarray}
The above suggests us to choose $$\lambda_0(\gamma) = 2x_{\max}\sigma \sqrt{\frac{\gamma^2 + 2 \log d}{\sum_{s=1}^{T}n_{s,w_k}}}.$$

\end{proof}

\subsection{Proof of Lemma \ref{SamCompa}}
\label{sec.SamComp}

\textbf{Lemma }\ref{SamCompa}
Given a sample set $\mathcal{A}_{w_k}$ satisfying template condition with rate $r$. Then
\begin{equation}
P(\widehat{\Sigma}(\mathcal{A}_{w_k}) \notin 
\mathcal{C}(\beta_{S_{w_k}}, \frac{\phi}{\sqrt{2}}\sqrt{\frac{|\mathcal{A}_{w_k}^{\sharp}|}{|\mathcal{A}_{w_k}|}})) \le \exp(-C_2(\phi_1)^2|\mathcal{A}_{w_k}^{\sharp}|).
\end{equation}
\begin{proof}
From our population assumption, the population covariance matrix $\Sigma_{w_k}$ satisfies the compatability condition $\Sigma_{w_k} \in \mathcal{C}(\beta_{S_{w_k}}, \phi).$
By carefully controlling $|\mathcal{A}_{w_k}|$ and  $|\mathcal{A}_{w_k}^{\sharp}|/|\mathcal{A}_{w_k}|$, one could first show $\|\Sigma_{w_k} - \widehat{\Sigma}(\mathcal{A}^{\sharp}_{w_k})\|_{\infty}\le \frac{\phi^2}{32|S_{w_k}|}$, which implies 
$\widehat{\Sigma}(\mathcal{A}^{\sharp}_{w_k}) \in \mathcal{C}(\beta_{S_{w_k}}, \frac{\phi}{\sqrt{2}})$ with high probability (by using Corollary 6.8 in page 152 of \cite{buhlmann2011statistics}). Next, by estimating an upper bound of the quadratic form induced by the covariance matrix of sample set $\mathcal{A}_{w_k}$, we can show, with high probability, that
\begin{equation*}
\widehat{\Sigma}(\mathcal{A}_{w_k}) \in \mathcal{C}(\beta_{S_{w_k}}, \frac{\phi}{\sqrt{2}}\sqrt{\frac{|\mathcal{A}_{w_k}^{\sharp}|}{|\mathcal{A}_{w_k}|}}).
\end{equation*}
\end{proof}
\section{Theory of LASSO}


\subsection{Basic Inequality}

\begin{lemma}{(Basic Ineuality from Optimality Condition)}\label{LASSOBasicIneq}
In LASSO,
\begin{equation}
\frac{1}{n}\|X(\hat{\beta}-\beta^0)\|_2^2
+
\lambda\|\hat{\beta}\|_1
\le
\frac{2}{n}\epsilon^\top X (\hat{\beta}-\beta^0)
+ 
\lambda\|\beta^0\|_1.
\end{equation}
\end{lemma}
\begin{proof}

To perform optimality analysis, we play with \text{true beta} $\beta^0$ and \text{empirical minimizer} $\hat{\beta}$(Short hand of $\widehat{\beta}_{w_k}(\mathcal{A}_{k}, \lambda)$). From the argument min, we start with
\begin{equation}
    \frac{1}{n}\|Y-X\hat{\beta}\|_2^2
    +\lambda \|\hat{\beta}\|_1
    \le
    \frac{1}{n}\|Y-X\beta^0\|_2^2
    +\lambda \|\beta^0\|_1
\end{equation}
Direct calculation gives us
\begin{eqnarray}
   && \frac{1}{n}\|Y-X\hat{\beta}\|_2^2
    -
    \frac{1}{n}\|Y-X\beta^0\|_2^2\\
    &=&
    \frac{1}{n}[Y^\top Y - 2Y^\top X\hat{\beta} + \hat{\beta}^\top X^\top X \hat{\beta}]
    -
    \frac{1}{n}[Y^\top Y - 2Y^\top X\beta^0 + (\beta^0)^\top X^\top X \beta^0]\\
    &=&
    \frac{1}{n}[2Y^\top X (\beta^0-\hat{\beta})
    +\hat{\beta}^\top X^\top X \hat{\beta}
    -(\beta^0)^\top X^\top X \beta^0]\\
    &=&
    \frac{1}{n}[2(X\beta^0 + \epsilon)^\top X (\beta^0-\hat{\beta})
    +\hat{\beta}^\top X^\top X \hat{\beta}
    -(\beta^0)^\top X^\top X \beta^0]\\
    &=&
    \frac{1}{n}[2(\beta^0)^\top X^\top X (\beta^0-\hat{\beta})
    +2\epsilon^\top X (\beta^0-\hat{\beta})
    +\hat{\beta}^\top X^\top X \hat{\beta}
    -(\beta^0)^\top X^\top X \beta^0]\\
    &=&
    \frac{1}{n}[
    2\epsilon^\top X (\beta^0-\hat{\beta})
    +(\beta^0)^\top X^\top X(\beta^0)
    -2(\beta^0)^\top X^\top X\hat{\beta}
    +\hat{\beta}^\top X^\top X \hat{\beta}]\\
    &=&
    \frac{1}{n}[
    2\epsilon^\top X (\beta^0-\hat{\beta})
    +
    (\beta^0-\hat{\beta})
    X^\top X
    (\beta^0-\hat{\beta})]\\
    &=&
    \frac{2}{n}\epsilon^\top X (\beta^0-\hat{\beta})
    +
    \frac{1}{n}\|X(\hat{\beta}-\beta^0)\|_2^2
\end{eqnarray}
\end{proof}

\begin{lemma}{(Basic Inequality on Good Event)}\label{GoodBasicIneq}
On good event $\mathcal{F}$ and with $\lambda \ge 2\lambda_0$, the basic inequality can be further reduced to 
\begin{equation}\label{LA_Check02}
 \frac{2}{n}\|X(\hat{\beta}-\beta^0)\|_2^2
 +\lambda
\|\hat{\beta}_{S_0^c}\|_1
\le 
3\lambda \|\hat{\beta}_{S_0}-\beta^0_{S_0}\|_1
\end{equation}
\end{lemma}

\begin{proof}

Recall the basic inequality
$$
\frac{1}{n}\|X(\hat{\beta}-\beta^0)\|_2^2
+
\lambda\|\hat{\beta}\|_1
\le
\frac{2}{n}\epsilon^\top X (\hat{\beta}-\beta^0) 
+ 
\lambda\|\beta^0\|_1
$$
Multiply it by $2$ to get
\begin{equation}
\frac{2}{n}\|X(\hat{\beta}-\beta^0)\|_2^2
+
2\lambda\|\hat{\beta}\|_1
\le
2\cdot \frac{2}{n}\epsilon^\top X (\hat{\beta}-\beta^0) 
+ 
2\lambda\|\beta^0\|_1
\end{equation}
Plug in the upper bound to get
\begin{equation}
\frac{2}{n}\|X(\hat{\beta}-\beta^0)\|_2^2
+
2\lambda\|\hat{\beta}\|_1
\le
2\cdot  
(\max_{j \in [p]}\frac{2}{n}|\epsilon^\top X^{(j)}|)
    \|\hat{\beta}-\beta^0\|_1
+ 
2\lambda\|\beta^0\|_1
\end{equation}
Then on good event $\mathcal{F}$, it becomes
\begin{equation}
\frac{2}{n}\|X(\hat{\beta}-\beta^0)\|_2^2
+
2\lambda\|\hat{\beta}\|_1
\le
2\cdot  \lambda_0
 \|\hat{\beta}-\beta^0\|_1
+ 
2\lambda\|\beta^0\|_1
\end{equation}
Apply $\lambda \ge 2\lambda_0$ to get
\begin{equation}\label{LA_check01}
\frac{2}{n}\|X(\hat{\beta}-\beta^0)\|_2^2
+
2\lambda\|\hat{\beta}\|_1
\le
\lambda
 \|\hat{\beta}-\beta^0\|_1
+ 
2\lambda\|\beta^0\|_1
\end{equation}

To further reduce equation (\ref{LA_check01}), we play with sparsity component. Let $S_0$ denote the sparsity location of truth $\beta^0$.

One the RHS, since $\beta^0_{S_0^c}=0$, we have an identity
\begin{equation}
\|\hat{\beta}-\beta^0\|_1
=
\|\hat{\beta}_{S_0}-\beta^0_{S_0}\|_1
+
\|\hat{\beta}_{S_0^c}\|_1    
\end{equation}

On the LHS, we have identity
\begin{equation}
    \|\beta^0\|_1
    =
     \|\beta^0_{S_0}\|_1+ \|\beta^0_{S_c}\|_1
    =
     \|\beta^0_{S_0}\|_1
\end{equation}
On the other hand, the empirical minimizer $\hat{\beta}$ only has identity
\begin{equation}
    \|\hat{\beta}\|_1
    =
     \|\hat{\beta}_{S_0}\|_1+ \|\hat{\beta}_{S_c}\|_1
\end{equation}
To link $\hat{\beta}_{S_0}$ with $\beta^0_{S_0}$ in $L_1$ norm, the inverse triangle inequality gives
\begin{equation}
    \|\hat{\beta}_{S_0}\|_1
    \ge
    \|\beta^0_{S_0}\|_1
    -
    \|\hat{\beta}_{S_0}-\beta^0_{S_0}\|_1
\end{equation}
Then we have an inequality
\begin{equation}
    \|\hat{\beta}\|_1
    \ge 
    \|\beta^0_{S_0}\|_1
    -
    \|\hat{\beta}_{S_0}-\beta^0_{S_0}\|_1
    +
    \|\hat{\beta}_{S_0^c}\|_1
\end{equation}

Combine these two observations into  equation (\ref{LA_check01}), we have inequality

\begin{equation}
\frac{2}{n}\|X(\hat{\beta}-\beta^0)\|_2^2
+
2\lambda
(\|\beta^0_{S_0}\|_1
    -
    \|\hat{\beta}_{S_0}-\beta^0_{S_0}\|_1
    +
    \|\hat{\beta}_{S_0^c}\|_1)
\le
\lambda
(\|\hat{\beta}_{S_0}-\beta^0_{S_0}\|_1
+
\|\hat{\beta}_{S_0^c}\|_1  )
+ 
2\lambda\|\beta^0_{S_0}\|_1
\end{equation}
Reorganize them into the inequality
\begin{equation}
\frac{2}{n}\|X(\hat{\beta}-\beta^0)\|_2^2
+
2\lambda
\|\hat{\beta}_{S_0^c}\|_1
\le
3\lambda \|\hat{\beta}_{S_0}-\beta^0_{S_0}\|_1
+
\lambda
\|\hat{\beta}_{S_0^c}\|_1
\end{equation}
\end{proof}

\begin{lemma}{(Compatibility passes $L_1$ norm to square root of $L_2$ norm)}\label{CompaIneq}

On good event $\mathcal{F}$, $\lambda \ge 2 \lambda_0$, and compatability condition associate with gram matrix $\hat{\Sigma}$ holds,
\begin{equation}
    \|\hat{\beta}_{S_0}-\beta^0_{S_0}\|_1
    \le
    \frac{\sqrt{s_0}}{\phi_0}\frac{1}{\sqrt{n}}\|X(\hat{\beta}-\beta^0)\|_2
\end{equation}
\end{lemma}
\begin{proof}
To further reduce the basic inequality on good event (\ref{LA_Check02}), we impose condition on sparsity component $S_0$. In lemma \ref{GoodBasicIneq}, two quantities we play with are $\|\hat{\beta}_{S_0^c}\|_1$ and $\|\hat{\beta}_{S_0}-\beta^0_{S_0}\|_1$.

An implication of equation (\ref{LA_Check02}) is, on good event $\mathcal{F}$, it is true that
\begin{equation}
\|\hat{\beta}_{S_0^c}-\beta^0_{S_0^c}\|_1
=
\|\hat{\beta}_{S_0^c}\|_1
\le 
3 \|\hat{\beta}_{S_0}-\beta^0_{S_0}\|_1,
\end{equation}
that is, on good event $\mathcal{F}$, the discrepancy between empirical minimizor and truth $\hat{\beta}-\beta^0$ always belongs to the class 
\begin{equation}
    \{\beta| \|\beta_{S_0^{c}}\|_1
    \le 3\|\beta_{S_0\|_1}\}.
\end{equation}
On such class, the \textbf{compatability condition} with Gram matrix $\hat{\Sigma}\equiv \frac{1}{n}X^\top X$ is
\begin{equation}
    \|\beta_{S_0}\|_1 
    \le
    \frac{\sqrt{s_0}}{\phi_0}
    \sqrt{\beta^\top\hat{\Sigma} \beta}
\end{equation}
Thus, we have
\begin{equation}
    \|\hat{\beta}_{S_0}-\beta^0_{S_0}\|_1
    \le
    \frac{\sqrt{s_0}}{\phi_0}
    \sqrt{(\hat{\beta}-\beta^0)^\top\hat{\Sigma} (\hat{\beta}-\beta^0)}
    =
    \frac{\sqrt{s_0}}{\phi_0}\frac{1}{\sqrt{n}}\|X(\hat{\beta}-\beta^0)\|_2
\end{equation}
\end{proof}

\subsection{Static Oracle Inequality}

\begin{theorem}{(Oracle Inequality of LASSO minimizor)}

On good event $\mathcal{F}$, $\lambda \ge 2 \lambda_0$, and compatability condition associate with gram matrix $\hat{\Sigma}$ holds,
\begin{equation}
\frac{1}{n}\|X(\hat{\beta}-\beta^0)\|_2^2
+\lambda \|\hat{\beta}-\beta^0\|_1
\le\frac{4s_0}{\phi_0^2}\lambda^2
\end{equation}
\end{theorem}
\begin{proof}
Plus both side of basic inequality on good event (\ref{LA_Check02}) an addition term $\lambda\|\hat{\beta}_{S_0}-\beta^0_{S_0}\|$ to get
\begin{equation}
\frac{2}{n}\|X(\hat{\beta}-\beta^0)\|_2^2
+\lambda \|\hat{\beta}-\beta^0\|_1
\le 4\lambda \|\hat{\beta}_{S_0}-\beta^0_{S_0}\|_1
\end{equation}
Input lemma \ref{CompaIneq} to get
\begin{equation}
    \frac{2}{n}\|X(\hat{\beta}-\beta^0)\|_2^2
+\lambda \|\hat{\beta}-\beta^0\|_1
\le 4\lambda\frac{\sqrt{s_0}}{\phi_0}\cdot\frac{1}{\sqrt{n}}\|X(\hat{\beta}-\beta^0)\|_2
\end{equation}
Set $u = \frac{1}{\sqrt{n}}\|X(\hat{\beta}-\beta^0)\|_2$ and $v = \lambda\frac{\sqrt{s_0}}{\phi_0}$. Note $(u-2v)^2 \ge 0$ implies $4uv \le u^2+4v^2$ to get
\begin{equation}
    \frac{2}{n}\|X(\hat{\beta}-\beta^0)\|_2^2
+\lambda \|\hat{\beta}-\beta^0\|_1
\le
\frac{\|X(\hat{\beta}-\beta^0)\|_2^2}{n}
+
4\lambda^2\frac{s_0}{\phi_0^2}
\end{equation}
Reorganize the terms to get
\begin{equation}
    \frac{1}{n}\|X(\hat{\beta}-\beta^0)\|_2^2
+\lambda \|\hat{\beta}-\beta^0\|_1
\le\frac{4s_0}{\phi_0^2}\lambda^2.
\end{equation}
\end{proof}

\section{Checking Optimal Allocation Condition}

Now we show two types of sample set--teamwork sample set and all sample set-produced from our proposed data collection protocol both satisfies the template condition.

The following lemmas are used to prove template condition of teamwork sample set and all sample set(lemma \ref{TeamTemplate} and lemma \ref{AllTemplate}).

\subsection{Teamwork Sample Set}

\begin{lemma}\label{TeamTemplate}
For any decision epoch $t \ge (Kq)^2$, the teamwork sample set for treatment $w_k$ up to time t, $\mathcal{D}_{[t],w_k}$, is a template sample set of rate $p_{*}$, with probability at least $1-\frac{2}{t^4}$.
\end{lemma}

\textit{Proof of Lemma \ref{TeamTemplate}.}
To check (i): 
Lemma \ref{TeamSize}, $q_0 \ge \frac{6\log d}{Np_*C_2^2(\phi_0)^2}$ and $t>(Kq)^2>3$ imply
\begin{equation}
|\mathcal{D}_{[t],k}|
\ge \frac{1}{2}Nq\log t
\ge 2Nq_0 
>\frac{6\log d}{p_*C_2^2(\phi_0)^2}.
\end{equation}
To check (ii): Lemma \ref{team2condi} shows that, for $t\ge (Kq)^2$, we have
\begin{equation}
    P(\frac{|\mathcal{D}_{[t],k}^{ \sharp}|}{|\mathcal{D}_{[t],k}|}\ge \frac{p_{*}}{2})\ge 1-\frac{2}{t^4}
\end{equation}

\begin{lemma}{(Size of Teamwork Sample Set)}\label{TeamSize}
If $t \ge (Kq)^2$, then $$\frac{1}{2}Nq \log t \le |\mathcal{D}_{[t],k}| \le 6 Nq \log t.$$
\end{lemma}
\begin{proof}
First we note
$$\mathcal{T}_{[t],k}
=\mathcal{T}_{\cdot, k} \cap [t]
=\cup_{n\ge 0}(\mathcal{T}_{n,k}\cap[t]).$$
At $t \in \mathcal{T}_{n,k}$, we have finished round $0,1,2,\cdots, n-1$ teamwork stage for arm k, each of size $Nq$, therefore
$$ nNq\le |\mathcal{D}_{[t],k}| \le (n+1)Nq.$$
With this, our task becomes to derive the lower bound and upper bound for $n$ and $n+1$ in terms of $\log t$ by using the condition $t \ge (Kq)^2$ 

For $t \in \mathcal{T}_{k,n}$,we have $$(2^n-1)Kq + 1 \le t \le (2^n)Kq,$$ which means
$$\log_2(\frac{t}{Kq}) \le n \le \log_2(\frac{t}{Kq}+1)+1.$$

Use condition $t \ge (Kq)^2$, one have $\log_2(Kq) \le \frac{1}{2}\log_2(t)$ and hence
$$n \ge \frac{1}{2}\log_2{t}.$$
On the other hand, we have
$$n+1 \le \log_2(\frac{t}{Kq}+1)+1 \le \frac{\log (2(t+\sqrt{t}))}{\log 2} \le 6 \log t.$$
\end{proof}


\begin{lemma}\label{team2condi}
If $t \ge (Kq)^2$, then
$P(\frac{|\mathcal{D}_{[t],k}^{ \natural}|}{|\mathcal{D}_{[t],k}|}\ge \frac{p_{*}}{2})\ge 1-\frac{2}{t^4}.$
\end{lemma}
\begin{proof}
Apply $P(|y-\mu|>\frac{\mu}{2}) < 2\exp[-0.1 \mu]$ in \cite{alon2004probabilistic} to the indicator random variable $I((i,s) \in \mathcal{D}_{[t],k}^{ \natural})$ for all $(i,s) \in \mathcal{D}_{[t],k}$ and using
$\mu = E[\sum_{(i,s) \in \mathcal{D}_{[t],k}}I[(i,s)\in \mathcal{D}_{[t],k}^{ \natural}]] \ge p_*|\mathcal{D}_{[t],k},|$
we get
$P(|\mathcal{D}_{[t],k}^{ \natural}|<\frac{p_*}{2}|\mathcal{D}_{[t],k}|)<
2e^{-\frac{p_*}{10}|\mathcal{D}_{[t],k}|}$
Therefore, by our control of the size of $|\mathcal{D}_{[t],k}|$ and the choice of $q_0$, we have
$P(|\mathcal{D}_{[t],k}^{ \natural}|<\frac{p_*}{2}|\mathcal{D}_{[t],k}|)
<2e^{-\frac{p_*}{5}q_0 \log t} \le \frac{2}{t^4}.$
\end{proof}

\subsection{All Sample Set}

We set $\mathcal{S} = \mathcal{T} \cup \mathcal{E}$ in this subsection.


\begin{lemma}\label{AllTemplate}
For any decision epoch $t \ge C_5$, the all sample set for treatment $w_k$ up to $t$, $\mathcal{S}_{[t],w_k}$, is a template sample set of rate $\frac{p_{*}}{2}$, with probability at least $1-\exp[-\frac{tp_{*}^{2}}{128}]$.
\end{lemma}

\textit{Proof of Lemma \ref{AllTemplate}}
To check (i):
Lemma \ref{TeamSize}, $q_0 \ge \frac{6\log d}{Np_*C_2^2(\phi_0)^2}$ and $t>C_5>3$ imply
\begin{equation}
|\mathcal{S}_{[t],k}| \ge |\mathcal{D}_{[t],k}|
\ge \frac{12\log d}{p_*C_2^2(\phi_0)^2} =  \frac{6\log d}{\frac{p_*}{2}C_2^2(\phi_0)^2}.
\end{equation}
To check (ii): Lemma \ref{all2condi} shows that, for $t\ge C_5$, we have
\begin{equation}
    P(\frac{
    |\mathcal{S}^{\sharp}_{[t],k}|
    }{
    |\mathcal{S}_{[t],k}|
    }
    \ge
    \frac{1}{2}\frac{p_*}{2})
    \ge
    1-\exp(-\frac{p_*^2}{128}\cdot t)
\end{equation}

\begin{lemma} \label{all2condi}
For $t > C_5$,
\begin{equation}
    P(\frac{
    |\mathcal{S}^{\sharp}_{[t],k}|
    }{
    |\mathcal{S}_{[t],k}|
    }
    \ge
    \frac{1}{2}\frac{p_*}{2})
    \ge
    1-\exp(-\frac{p_*^2}{128}\cdot t)
\end{equation}
\end{lemma}
\begin{proof}
We start from noting the fact that the all sample set for treatment $w_k$, $\mathcal{S}_{[t],k}$, can have at most $t$ elements up to time t ($|\mathcal{S}_{[t],k}|\le t$) implies
\begin{eqnarray}
    &&
    P(\frac{
    |\mathcal{S}^{\sharp}_{[t],k}|
    }{
    |\mathcal{S}_{[t],k}|
    }
    <
    \frac{1}{2}\frac{p_*}{2})
    \ge
    P(\frac{
    |\mathcal{S}^{\sharp}_{[t],k}|
    }{
    t
    }
    <
    \frac{p_*}{4})
    =
    P(|\mathcal{S}^{\sharp}_{[t],k}| 
    < \frac{p_*}{4} \cdot t)
\end{eqnarray}

To handle RHS, we note that the size of $\mathcal{S}_{[t],k}^{ \sharp}$ admits a representation 
\begin{equation}
     |\mathcal{S}_{[t],k}^{ \sharp}|= \sum_{s=1}^{t}\sum_{i \in N(s)}I((X_{(i,s)}, Y_{(i,s)})\in \mathcal{S}_{[t],k}^{ \sharp}).
\end{equation}
The strategy to utilize such representation is first to construct a martingale difference sequence and then apply Azuma's inequality to attain desired result.

First, all samples been collected in $\mathcal{S}_{[t],k}^{ \sharp}$ are  optimal allocation in selfish stage given good event happens. Thus, whether a sample $(X_{i,s}, Y_{i,s})$ belongs to $\mathcal{S}_{[t],k}^{ \sharp}$ has a representation
\begin{equation}
I((X_{i,s}, Y_{i,s})\in S^{\sharp}_{[t],k})
=
I(E_{s-1})I(X_{(i,s)}\in U_{w_k})I(s \notin \mathcal{T}_{[t], \cdot}).
\end{equation}

Recall that samples in $S^{\sharp}_{[s],k}$ also satisfies model assumption and hence can be written as $Y = X^\top \beta + \epsilon$.
Let $\mathcal{G}_{s}$ be the sigma algebra generated by the first $N(s)\equiv |S^{\sharp}_{[s],k}|$ 
rows of the design matrix $X$ and the first $N(s)$ entries of the noise vector $\epsilon$, and let $\mathcal{G}_{0} = \phi$. With this, $I(E_{s-1})$ is $\mathcal{G}_{s-1}$ measurable; $I(X_{(i,s)} \in U_{w_k})$ is $\mathcal{G}_{s}$ measurable and independent of $\mathcal{G}_{s-1}$; $I(s \notin \mathcal{T}_{[t],\cdot})$ is deterministic by planning of teamwork stage. Follow the Doob's martingale construction, define
\begin{equation}
    M_{s}=E[|\mathcal{S}_{[t],k}^{ \natural}||\mathcal{G}_{s}]
\end{equation} for all $s\in[t]\cup\{0\}$. The resulting sequence $M_0, M_1,\cdots, M_t$ is a martingale adapted to the filtration $\mathcal{G}_{0}\subset \mathcal{G}_{1} \subset \cdots \mathcal{G}_{t}$ with $M_0 = E[|\mathcal{S}_{[t],k}^{ \natural}|]$ and $M_t=|\mathcal{S}_{[t],k}^{ \natural}|.$ The desired martingale differences is thus $M_{s} - M_{s-1}$.

Now since the martingale differences $M_{s}-M_{s-1}$ are bounded by $N(s)-N(s-1)$, the Azuma's inequality, (see. Theorem 7.2.1 from Alon and Spencer 1992), to obtain for all $\eta >0,$
\begin{equation}
P(|\mathcal{S}_{[t],k}^{ \natural}|
<E(|\mathcal{S}_{[t],k}^{ \natural}|)
-\eta)
\le \exp(-\frac{\eta^2}{2N(t)}).
\end{equation}
Now a lower bound for expected size of $\mathcal{S}_{[t],k}^{ \natural}$ follows from adopted policy that
\begin{align*}
E[|\mathcal{S}_{[t],k}^{\sharp}|]
&=
\sum_{s=1}^{t}\sum_{i \in N(s)}P((X_{(i,s)}, Y_{(i,s)})\in \mathcal{S}_{[t],k}^{ \sharp})
\ge[t-|\mathcal{T}_{[t],\cdot}|-(Kq)^2]\frac{p_{*}}{2}\\
&\ge
[t-6KNq\log t-(Kq)^2]\frac{p_{*}}{2}
\ge\frac{3p_{*}}{8}t,
\end{align*}
where the last inequality from the definition of constant $C_5$. Thus, taking $\eta = \frac{p_*}{8}t,$ we have
\begin{equation}
P(|\mathcal{S}_{[t],k}^{ \sharp}|< \frac{p_{*}}{4}t)\le \exp(-\frac{p_{*}^2}{128}t).
\end{equation}
\end{proof}

\section{Deviation Inequalities of Teamwork LASSO and ALL LASSO}

\subsection{Teamwork LASSO-Proof in Corollary \ref{TeamworkOracle}}

\textit{Proof.}
Note $C_1(\frac{\phi_1 \sqrt{r}}{2}) = \frac{r^2}{16}C_1(\phi_1)$.
Apply Theorem \ref{TeamSize} for $\chi= \frac{h}{4x_{\max}}$ and $r=p_*$. First, $q_0 \ge \frac{512 x_{\max}^2}{NC_1(\phi_1)p_*^2 h^2}$ and lemma \ref{TeamSize} imply
\begin{eqnarray}
&&-C_1(\frac{\phi_1 \sqrt{p_*}}{2})
|\mathcal{D}_{[t],w_k}|\chi^2
+\log d
\le
\frac{NC_1(\phi_1)p_*^2 h^2}{128 x_{\max}^2}\log t \cdot q_0 \le -4\log t.
\end{eqnarray}
Second, $|\mathcal{A}^{\sharp}_{w_k}| \ge \frac{r}{2}|\mathcal{A}_{w_k}|\ge \frac{p_*}{2}Nq_0$ and $q_0 \ge \frac{8}{NC_2(\phi_1)^2p_*}$ and it implies
\begin{equation}
-|\mathcal{D}_{[t],w_k}^{\sharp}|C_2(\phi_1)^2
\le - \frac{NC_2(\phi_1)^2p_*}{2}\cdot q_0 \le -4\log t
\end{equation}
Last, we find
\begin{eqnarray}
&&P(\|\hat{\beta}_{w_k}(\mathcal{D}_{[t],k}, \lambda_1)
-\beta_{w_k}\|_1 >\frac{h}{4x_{\max}})\\
&\le&
P(\|\hat{\beta}_{w_k}(\mathcal{D}_{[t],k}, \lambda_1)
-\beta_{w_k}\|_1 >\frac{h}{4x_{\max}}, \frac{|\mathcal{D}_{[t],k}^{ \sharp}|}{|\mathcal{D}_{[t],k}|}\ge \frac{p_{*}}{2})
+
P(\frac{|\mathcal{D}_{[t],k}^{ \sharp}|}{|\mathcal{D}_{[t],k}|}< \frac{p_{*}}{2})\\
&\le& 2\cdot 
\frac{1}{t^4} + \frac{1}{t^4}+ \frac{2}{t^4} =\frac{5}{t^4} 
\end{eqnarray}

\subsection{All LASSO--Proof in Corollary \ref{AllOracle}}

\textit{Proof.}
Note $C_1(\frac{\phi_1 \sqrt{r}}{2}) = \frac{r^2}{16}C_1(\phi_1)$.
Apply Theorem \ref{MainOra} for $\chi= \frac{16}{\sqrt{p_*^3 C_1(\phi_0)}}
\sqrt{
\frac{\log t + \log d}{t}
}$ and $r=\frac{p_*}{2}$.
First, $|\mathcal{S}_{[t],w_k}| \ge \frac{p_* t}{4}$ and lemma \ref{TeamSize} imply
\begin{equation}
-C_1(\frac{\phi_1 \sqrt{p_*/2}}{2})
|\mathcal{S}_{[t],w_k}|\chi^2
+\log d
\le
-\frac{p_*^2}{64}C_1\cdot \frac{p_* t}{4}\cdot 256\frac{\log t + \log d}{tp_*^3 C_{1}}
+\log d
=-\log t.
\end{equation}
Second, $|\mathcal{S}_{[t],w_k}^{\sharp}| \ge \frac{p_* }{4}|\mathcal{S}_{[t],w_k}| \ge \frac{p_*^2 t}{16}$ and $C_2^2 \ge \frac{1}{2}$ imply
\begin{equation}
-|\mathcal{S}_{[t],w_k}^{\sharp}|C_2(\phi_1)^2
\le - \frac{p_*^2}{32}\cdot t
\end{equation}
Last, we find
\begin{eqnarray}
&&P(\|\hat{\beta}_{w_k}(\mathcal{D}_{[t],k}, \lambda_1)
-\beta_{w_k}\|_1 >\frac{16}{\sqrt{p_*^3 C_1(\phi_0)}}
    \sqrt{
    \frac{\log t + \log d}{t}
    })\\
&\le&
P(\|\hat{\beta}_{w_k}(\mathcal{D}_{[t],k}, \lambda_1)
-\beta_{w_k}\|_1 >\frac{16\sqrt{
    \frac{\log t + \log d}{t}
    }}{\sqrt{p_*^3 C_1(\phi_0)}}
    , \frac{|\mathcal{S}_{[t],k}^{ \sharp}|}{|\mathcal{S}_{[t],k}|}\ge \frac{p_{*}}{2})
+
P(\frac{|\mathcal{S}_{[t],k}^{ \sharp}|}{|\mathcal{S}_{[t],k}|}< \frac{p_{*}}{2})\\
&\le& 2\cdot 
\frac{1}{t} + \exp(-\frac{p_*^2}{32}\cdot t)
+\exp(-\frac{p_*^2}{128}\cdot t)
\le
2(\frac{1}{t}+\exp(-\frac{p_*^2}{32}\cdot t))
\end{eqnarray}

\section{Regret Analysis}

We show the properties of $\hat{K}(x)$ for $x \in \mathcal{X}$ and for $x \in U_{w}$ of a available treatment $w$. 
In words, for any given observed covariate $x \in \mathcal{X}$, Teamwork LASSO excludes those sub-optimal treatment of $x$ up to tolerance level $h$. If $x \in U_{w}$, then Teamwork excludes all treatment other than the optimal treatment of $x$. Therefore, the probability of random covariate $X$ belongs to $U_{w}$ matters.

\subsection{Proof of lemma \ref{allCandi}}

\begin{lemma}{(For $x \in 
\mathcal{X}$)}\label{allCandi}
Suppose the ($t-1$)th decision epoch is in selfish stage and event $E_{t-1}$ holds. Then for each available treatment $w_i \in \mathcal{W}$ and any possible observed covariate $x \in \mathcal{X}$, the estimated optimal treatment candidate set contains the optimal treatment of $x$: $w^*(x) \equiv \arg\max_{w \in \mathcal{W}} \langle x, \beta_{w}
\rangle$
and no any sub-optimal treatment $w \in \mathcal{W}_{\text{sub}}$. That is,
\begin{equation}
    w^*(x) \in \widehat{K}(x)~~~\text{ and }~~~
    w \notin \widehat{K}(x) ~~\text{for all}~~w \in \mathcal{W}_{\text{sub}}
\end{equation}
\end{lemma}

\begin{proof}
First, we show $ w^*(x) \in \widehat{K}(x)$. Note at the $t$th decision epoch, the optimal treatment suggested by Teamwork LASSO is $w^{\text{team}}(x) \equiv \arg\max_{w \in \mathcal{W}} x^\top \widehat{\beta}_{w}(\mathcal{D}_{[t-1], w}, \lambda_1)$.
Since $E_{t-1}$ holds, it implies $x^\top 
\widehat{\beta}_{w}(\mathcal{D}_{[t-1], w}, \lambda_1) - x^\top \beta_{w} < x_{\max}\cdot \frac{h}{4 x_{\max}}=\frac{h}{4}$ for all available treatment $w$, which includes $w^{*}$ and $w^{\text{team}}$. 
\begin{eqnarray}
&&
x^\top \widehat{\beta}(\mathcal{D}_{[t-1], w^{\text{team}}})
-
x^\top \widehat{\beta}(\mathcal{D}_{[t-1], w^{*}})\\
&=&
(x^\top\widehat{\beta}(\mathcal{D}_{[t-1], w^{\text{team}}})
-x^\top \beta_{w^{\text{team}}})
+
(x^\top \beta_{w^{\text{team}}}
-
x^\top\beta_{w^{*}})
+
(x^\top\beta_{w^{*}}
-
x^\top\widehat{\beta}(\mathcal{D}_{[t-1], w^{*}})\\
&\le& \frac{h}{4} + 0 + \frac{h}{4} = \frac{h}{2},
\end{eqnarray}
where the last inequality is from the definition of $w^{*}(x)$ that $x^\top\beta_{w^{*}}- x^\top\beta_{w^{\text{team}}}<0$.

Second, we show $w_{\text{sub}} \notin \widehat{K}(x) ~~\text{for all}~~w_{\text{sub}} \in \mathcal{W}_{\text{sub}}$. Since $E_{t-1}$ holds, it implies $x^\top 
\widehat{\beta}_{w}(\mathcal{D}_{[t-1], w}, \lambda_1) - x^\top \beta_{w} > -x_{\max}\cdot \frac{h}{4 x_{\max}}=-\frac{h}{4}$ for all available treatment $w$, which includes $w^{\text{team}}$ and $w^{\text{sub}}$.
\begin{eqnarray}
&&
x^\top \widehat{\beta}(\mathcal{D}_{[t-1], w^{\text{team}}})
-
x^\top \widehat{\beta}(\mathcal{D}_{[t-1], w^{\text{sub}}})\\
&\ge&
x^\top \widehat{\beta}(\mathcal{D}_{[t-1], w^{*}})
-
x^\top \widehat{\beta}(\mathcal{D}_{[t-1], w^{\text{sub}}})\\
&=&
(x^\top\widehat{\beta}(\mathcal{D}_{[t-1], w^{*}})
-x^\top \beta_{w^{*}})
+
(x^\top \beta_{w^{*}}
-
x^\top\beta_{w^{\text{sub}}})
+
(x^\top\beta_{w^{\text{sub}}}
-
x^\top\widehat{\beta}(\mathcal{D}_{[t-1], w^{\text{sub}}})\\
&\ge& -\frac{h}{4} + h + -\frac{h}{4} = \frac{h}{2},
\end{eqnarray}
where the last inequality is from the definition of $\mathcal{W}_{\text{sub}}$ that $x^\top\beta_{w^{*}}- x^\top\beta_{w^{\text{sub}}}>h$.
\end{proof}

\subsection{Proof of lemma \ref{optCandi}}

\begin{lemma}(For $x \in 
{U}_{w_{i}}$)\label{optCandi}
Suppose the ($t-1$)th decision epoch is in selfish stage and event $E_{t-1}$ holds. Then for each available treatment $w_i \in \mathcal{W}$, if a observed covariate $x$ belongs to $U_{w_i}$, then the estimated optimal treatment candidate set contains only treatment $w_i$, that is
\begin{equation}
    \widehat{K}(x) = \{w_i\}.
\end{equation}
\end{lemma}

\begin{proof}
For every treatment $w_j$ other than $w_i$, 
since $x \in U_{w_i}$, definition of $U_{w_i}$ implies $x^\top \beta_{w_i} - x^\top \beta_{w_j} > h$; since $E_{t-1}$ holds, it implies $x^\top 
\widehat{\beta}_{w}(\mathcal{D}_{[t-1], w}, \lambda_1) - x^\top \beta_{w} > -x_{\max}\cdot \frac{h}{4 x_{\max}}=-\frac{h}{4}$. Combine them to obtain, for every treatment $w_j$ other than $w_i$
\begin{eqnarray}
&&
x^\top \widehat{\beta}_{w_i}(\mathcal{D}_{[t-1], w_i}, \lambda_1)
-
x^\top \widehat{\beta}_{w_j}(\mathcal{D}_{[t-1], w_j}, \lambda_1)\\
&=&
x^\top [\widehat{\beta}_{w_i}(\mathcal{D}_{[t-1], w_i}, \lambda_1) -\beta_{w_i}]
-
x^\top [\widehat{\beta}_{w_j}(\mathcal{D}_{[t-1], w_j}, \lambda_1)-\beta_{w_j}]
+
x^\top[\beta_{w_i}-\beta_{w_j}]\\
&\ge& -\frac{h}{4}-\frac{h}{4}+h 
= \frac{h}{2}.
\end{eqnarray}
That is, for every treatment $w_j$ other than $w_i$,
\begin{equation}
x^\top \widehat{\beta}_{w_i}(\mathcal{D}_{[t-1], w_i}, \lambda_1)
\ge
x^\top \widehat{\beta}_{w_j}(\mathcal{D}_{[t-1], w_j}, \lambda_1) + \frac{h}{2}.
\end{equation}
Therefore, by construction of optimal treatment candidate set, we conclude $\widehat{K}(x) = \{w_i\}$.
\end{proof}

\subsection{Regret bound for case (4)}
\label{case4regret}

\begin{lemma}\label{techRegret}
\begin{equation}
    f(t) = 
    [4Kbx_{\max} + C_3(\phi_0, p_{*})\cdot \log d]\frac{1}{t}
    +
    8Kbx_{\max}
    \exp[-\frac{p_*^2C_2(\phi_0)^2}{32}\cdot t]
    +
    C_3(\phi_0, p_{*})
    \frac{\log t}{t}
\end{equation}
\end{lemma}

\begin{proof}
Without loss of generality, for a observed covariate vector  $x_{i,t}$ of the $i$th user at the $t$th decision epoch, assume $w_1$ is the optimal treatment, that is $x_{i,t}^\top\beta_{w_1} = \max_{w \in \mathcal{W}}x_{i,t}^\top \beta_{w}$. First, we note that the instantaneous regret occurs if we allocate treatment other than $w_1$ to covariate $x$. This happens when $x^\top \widehat{\beta}(\mathcal{S}_{[t-1], w}) > x^\top \widehat{\beta}(\mathcal{S}_{[t-1], w_1})$ for some treatments $w$. This observation suggests
\begin{eqnarray}
r_{i,t}
&=&
E[\sum_{w_k\in \hat{K}(x_{i,t})}x_{i,t}^\top(\beta_{w_1} - \beta_{w_k})I(\pi(x_{i,t} = w_k)]\\
&\le&
E[\sum_{w_k\in \hat{K}(x_{i,t})}x_{i,t}^\top(\beta_{w_1} - \beta_{w_k})
I(
x_{i,t}^\top
\widehat{\beta}(\mathcal{S}_{[t],w_{k}})
>
x_{i,t}^\top
\widehat{\beta}(\mathcal{S}_{[t],w_{1}})
]
\end{eqnarray}
Second, to handle RHS, define a function $g(x) \equiv x^\top (\beta_{w_1}-\beta_{w_k})$ consider the set
\begin{equation}
    B_{w_k}\equiv\{x|x^\top(\beta_{w_1}-\beta_{w_k})>2\delta x_{\max}\}.
\end{equation}
The boundness assumption on observed covariate $x$ and efficacy parameter $\beta_{w}$ suggests $g(x) \le 2bx_{\max}$ for all $x \in B_{w_k}$; the definition of set $B_{w_k}$ suggests $g(x) \le 2\delta x_{\max}$ for all $x \in B_{w_k}^{c}$. This observation suggests 
\begin{eqnarray}
r_{i,t}
&\le& 
|\hat{K}(x_{i,t})|
\cdot
2b x_{\max}
\cdot 
E[I(
x_{i,t}^\top
\widehat{\beta}(\mathcal{S}_{[t],w_{k}})
>
x_{i,t}^\top
\widehat{\beta}(\mathcal{S}_{[t],w_{1}})I(x_{i,t} \in B_{w_k})]\\
&+&
|\hat{K}(x_{i,t})|
\cdot
2\delta x_{\max}
\cdot 
E[I(
x_{i,t}^\top
\widehat{\beta}(\mathcal{S}_{[t],w_{k}})
>
x_{i,t}^\top
\widehat{\beta}(\mathcal{S}_{[t],w_{1}})I(x_{i,t} \in B_{w_k}^{c})]\\
&\le&
K2b x_{\max}
E[I(
x_{i,t}^\top
\widehat{\beta}(\mathcal{S}_{[t],w_{k}})
>
x_{i,t}^\top
\widehat{\beta}(\mathcal{S}_{[t],w_{1}})I(x_{i,t}^\top (\beta_{w_1}-\beta_{w_k})>2\delta x_{\max})] \label{keyeq01}\\
&+&
K2\delta x_{\max}
E[I(x_{i,t}^\top (\beta_{w_1}-\beta_{w_k})\le 2\delta x_{\max})]\label{keyeq02}
\end{eqnarray}
Third, we handle equation (\ref{keyeq01}) and (\ref{keyeq02}). We note the marginal condition implies
\begin{equation}
(\ref{keyeq02})
=K2\delta x_{\max} P(X^\top (\beta_{w_1}-\beta_{w_k})\le 2\delta x_{\max})
\le C_0\cdot 2\delta x_{\max}.
\end{equation}
Based on this observation, we have
\begin{eqnarray}
(\ref{keyeq01})
&\le&
K2bx_{\max}\cdot (P(\|\beta_{w_1}-\widehat{\beta}_{w_1}(\mathcal{S}_{[t],w_1})\|_1>\delta)
+
P(\|\widehat{\beta}_{w_k}(\mathcal{S}_{[t],w_k})-\beta_{w_k}\|_1>\delta))\\
&\le&
K2bx_{\max}\cdot 2 \cdot (\frac{1}{t} + 2\exp(-\frac{p_*^2 C_2(\phi_0)^2}{32}\cdot t))
\end{eqnarray}
Last, combine above results and take $\delta = 16\sqrt{\frac{\log t + \log d}{p_{*}^3 C_1 t}}$, we have
\begin{eqnarray}
&&r_{i,t}\\
&\le&
K2bx_{\max}\cdot 2 \cdot (\frac{1}{t} + 2\exp(-\frac{p_*^2 C_2(\phi_0)^2}{32}\cdot t))
+
2\delta x_{\max}\cdot C_0\cdot 2\delta x_{\max})\\
&=&
K4bx_{\max} (\frac{1}{t} + 2\exp(-\frac{p_*^2 C_2(\phi_0)^2}{32}\cdot t))
+
4\delta^2 x_{\max}^2\cdot C_0)\\
&=&
[4Kbx_{\max} + C_3(\phi_0, p_{*}) \log d]\frac{1}{t}
    +
    8Kbx_{\max}
    \exp[-\frac{p_*^2C_2(\phi_0)^2}{32}\cdot t]
    +
    C_3(\phi_0, p_{*})
    \frac{\log t}{t},
\end{eqnarray}
as desired.
\end{proof}


\subsection{Full Regret Bound--Proof of Theorem \ref{thm:regret_upper}}

The regret can be bounded by:
\begin{eqnarray}
&& 
R_{T}=
\sum_{t \in [T]}\sum_{i \in [N]}r_{i,t}\\
&=&
\sum_{t \in [C_5]}\sum_{i \in [N]}r_{i,t}
+
\sum_{t \in [C_5: T] \cap \mathcal{T}}\sum_{i \in [N]}r_{i,t}
+
\sum_{t \in [C_5: T] \cap \mathcal{T^{c}}}\sum_{i \in [N]}r_{i,t}\\
&\le&
N\cdot C_5 \cdot 2bx_{\max}
+
N\cdot |\mathcal{T}| \cdot 2bx_{\max}
+N \cdot 
\sum_{t \in [C_5: T] \cap \mathcal{T^{c}}}[\frac{K}{t^4}\cdot 2bx_{\max}+f(t)]\\
&\le& N C_5  2bx_{\max}
+
N (6q\log T K)  2bx_{\max}
+
NK2bx_{\max}
\int_{1}^{T}\frac{1}{t^4}dt
+
N\cdot \int_{1}^{T}f(t)dt\\
&\le&
N\cdot \{2bx_{\max}\cdot [C_5 + 6qK \log T + K]\\
&+&
[4Kbx_{\max} + C_3(\phi_0, p_{*})\cdot \log d] \log T
    +
    8Kbx_{\max} C_4
    +
    C_3(\phi_0, p_{*})
    (\log T)^2\}
\end{eqnarray}

\section{Constants}
Here we list the constants that appear in the proof.
\begin{itemize}
\item 
$C_1(\phi_0)\equiv \frac{\phi_0^4}{512 s_0^2\sigma^2 x_{\max}^2}$ 
\item $C_2\equiv \min\{\frac{1}{2}, \frac{\phi_0^2}{256 s_0 x_{\max}^2}\}$
\item
$C_3\equiv \frac{1024 K C_0 x_{\max}^2}{p_*^3 C_1}$
\item
$C_4 \equiv \frac{8Kb x_{\max}}{1-\exp(-\frac{p_*^2}{32})}$
\item
$C_5 \equiv \{t \in \mathbb{Z}^{+}| t \ge 24Kq \log t + 4(Kq)^2\}$
\item
$q_0 \equiv \max\{\frac{20}{Np_{*}}, \frac{4}{Np_* C_2^2}, \frac{3\log d}{Np_*C_2^2}, \frac{1024 x_{\max}^2\log d}{Nh^2p_*^2 C_1}\}$.

\end{itemize}

\section{Experiment} \label{append:experiment}

In Figure. \ref{fig:teamworkLasso_exp}, we compare our \texttt{Teamwork LASSO Bandit} with batch size $N=4$ and $N=12$
to the \texttt{LASSO Bandit}  in \cite{bastani2020online}. In the attached plot, covariate dimension d = 200, 500 and 1000, number of treatments (arms) K =3, the length of exploration phase q = 1,2,3,4,5,6 with a total number of decisions 5000. N is the batch size, where N=1 corresponds to LASSO Bandit and N=4, 12 corresponds our Teamwork LASSO Bandit. We run 100 replications for each setting.

\textbf{Remark on cumulative regret and covariate vector dimension.}
In the experiment, we increase the covariate vector dimension from 200, 500 to 1000. The performance of high update frequency algorithm is more sensitive to the increase in covariate dimension than our low update frequency algorithm.

\textbf{Remark on the length of exploration phase $q$.}
In real world practice, the length of exploration phase q is pre-specified and then an explore-exploitation policy follows the choice of q. Given the same total number of decisions, it is often the case that one prefers a smaller value of q, which means fewer regret from exploration and is more time efficient in the sense that more rounds of explore-exploit can be done.

\begin{figure*}[b]
\vskip 0.2in
\begin{center}
\centerline{\includegraphics[width=0.9\textwidth]{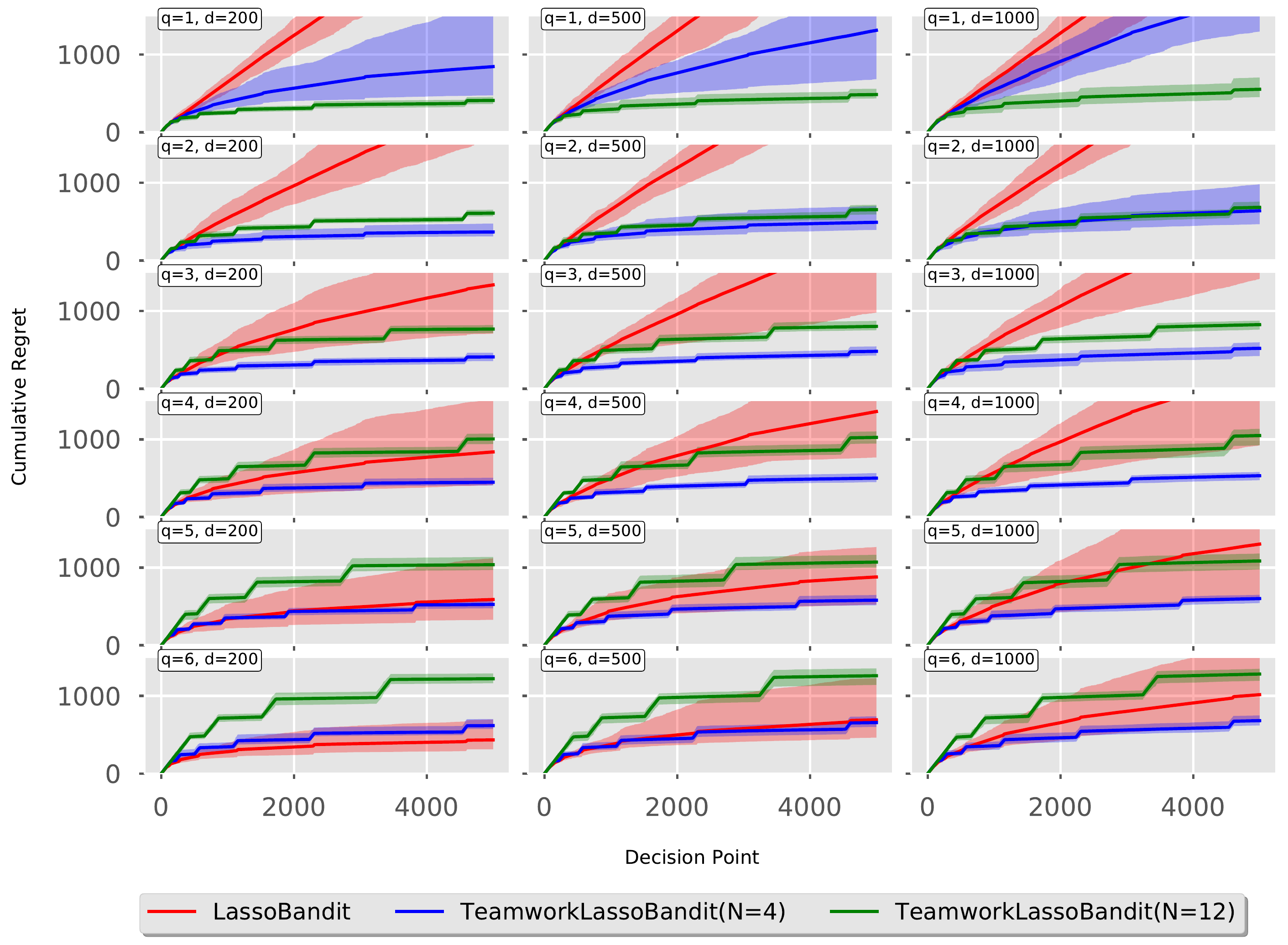}}
\caption{Comparison of our \texttt{Teamwork LASSO Bandit} with batch size $N=4$ and $N=12$
to the \texttt{LASSO Bandit}  in \cite{bastani2020online}. The error bars represent the maximum and minimum of the regret among 100 replications.}
\label{fig:teamworkLasso_exp}
\end{center}
\vskip -0.2in
\end{figure*}

\end{document}